\documentclass[10pt,twocolumn,letterpaper]{article}

\pdfoutput=1

\usepackage{cvpr}

\usepackage[pdftex]{graphicx}

\usepackage{amsmath}
\usepackage{amssymb}
\usepackage{amsthm} 

\usepackage{txfonts}
\cvprfinalcopy 

\def\cvprPaperID{259} 

\ifcvprfinal\fi
\begin{document}

\title{Self-calibration-based Approach to Critical Motion Sequences of\\
Rolling-shutter Structure from Motion}

\author{Eisuke Ito and Takayuki Okatani\\
Graduate School of Information Sciences, Tohoku University\\
{\tt\small \{ito,okatani\}@vision.is.tohoku.ac.jp}
}

\def\mat#1{\mathtt #1}
\def\vec#1{\mathbf #1}
\def\gvec#1{\boldsymbol #1}
\def\Eq#1{Eq.(\ref{eqn:#1})}
\def\Fig#1{Fig.\ref{fig:#1}}
\newcommand{\tfbox}[1]{\fboxsep=0pt\fbox{#1}}
\newtheorem{prop}{Proposition}

\maketitle

\begin{abstract}
In this paper we consider critical motion sequences (CMSs) of
rolling-shutter (RS) SfM. Employing an RS camera model with linearized
pure rotation, we show that the RS distortion can be approximately
expressed by two internal parameters of an ``imaginary'' camera plus
one-parameter nonlinear transformation similar to lens distortion. We
then reformulate the problem as self-calibration of the imaginary
camera, in which its skew and aspect ratio are unknown and varying in
the image sequence. In the formulation, we derive a general
representation of CMSs. We also show that our method can explain the
CMS that was recently reported in the literature, and then present a
new remedy to deal with the degeneracy. Our theoretical results agree
well with experimental results; it explains degeneracies observed when
we employ naive bundle adjustment, and how they are resolved by our
method. 
\end{abstract}

\section{Introduction}

Most consumer imaging devices such as smartphones and action cameras
use CMOS sensors, the majority of which use a rolling-shutter. 
In recent years, increasing attention has been paid to performing SfM
(structure-from-motion) from multi-view images with distortion created
by a rolling shutter (RS). A number of methods have been proposed that
can deal with RS distortion in several different problem settings
\cite{meingastgeometric2005,ait-aidersimultaneous2006,hedborgstructure2011,forssénrectifying2010,alblr6p-rolling2015}. The
standard formulation is RS bundle adjustment (BA), which is to assume
a model of RS distortion and optimize its model parameters together
with standard camera parameters in the framework of bundle adjustment
\cite{hedborgstructure2011,hedborgrolling2012,saurer2016cvpr}.

Despite the many studies of RS SfM, there were few studies about its
degeneracy. An only exception is the recent study of Albl et
al. \cite{Albleccv2016}. They showed that there is a degenerate case
in RS SfM and then they explained the mechanism of why it
arises. However, the study only deals with a single, {\em specific}
case of degeneracy. Thus, there remain several open questions, for
instance, {\em are there more degeneracy and critical cases?}

A central motivation behind this study is to answer such questions. We
wish to derive {\em general} representation of degeneracy in RS
SfM. However, it is not a simple task, since RS SfM is a highly
nonlinear problem with many unknowns. To mitigate this difficulty, in
this paper, we limit our attention to a simple but practically
important type of RS distortion, and then borrow the classical
formulation of self-calibration of a camera, in which CMSs were
previously studied extensively. To be specific, we first assume that
camera motion inducing the RS distortion is pure rotation with a
constant angular velocity, which may be independent of the camera pose
of the viewpoint. Further assuming that the rotation is small and can
be linearized, we then show that the RS distortion can be expressed
by two internal parameters, specifically, skew and aspect ratio, of an
{\em imaginary} camera, along with a parameter of a nonlinear
transformation that is similar to (and may be regarded as a special
type of) lens distortion. Assuming this model, we show that RS SfM can
be recasted as self-calibration of the imaginary camera in the case
where the above parameters are unknown and varying in an image
sequence. In this self-calibration formulation, we will derive a
general representation of critical motion sequences (CMSs). We also
derive an explicitly represented CMS, which coincides with the one
derived in \cite{Albleccv2016}, with a new remedy to deal with it. We
will finally show the usefulness and effectiveness of our approach
through a series of experimental results.



\section{Related work}


A number of studies have been conducted to deal with RS distortion in
single or multi-view images. Depending on the type of problems, they
are classified into absolute pose problem
\cite{ait-aidersimultaneous2006,magerand2012eccv,alblr6p-rolling2015,saurer2015icirs},
relative pose problem \cite{dai2016cvpr}, multi-view optimization
(bundle adjustment)
\cite{hedborgrolling2012,hedborgstructure2011,saurer2016cvpr}, stereo
\cite{ait-aiderstructure2009,saurerrolling2013}, and
rectification/stabilization
\cite{forssénrectifying2010,ringabyefficient2012,jiaprobabilistic2012,grundmanncalibration-free2012}.

The most closely related to our study is the study of Albl et
al. \cite{Albleccv2016}. Their study is the first to show that there
is degeneracy in RS SfM. Specifically, employing the most widely used
RS model with linearized rotation and constant velocity translation,
they showed the existence of a CMS. The CMS is that {\em ``all images
  are captured by cameras with identical y (i.e., RS readout)
  direction in space''}. They proved that these images can also be
explained by RS cameras and a scene all lying in a single plane,
meaning degeneracy.

Our study differs from theirs in that we derive a {\em general}
representation of CMSs, although we assume an RS model of
rotation-only camera motion. We also show that our method can explain
the same CMS as \cite{Albleccv2016} but in a different way, and then
present a new method to cope with it. Another difference is that our
derivation needs approximation beyond the approximation of linearized
rotation. Thus, rigorously speaking, our results state that, at least
in the range where our approximation is effective, RS SfM will suffer
from the same degeneracy. However, considering the nature of the
employed approximations, it will probably have at least instability
beyond this range. These agree well with our experimental results that
will be shown later.

As mentioned earlier, the central idea of our approach is to formulate
the problem as self-calibration of cameras. Self-calibration of cameras
was studied extensively from 1990's to early 2000's
\cite{maybanka1992,heydeneuclidean1997,pollefeysself-calibration1999,heydenminimal1998,heydenflexible1999}.
It had been shown that when internal camera parameters are all unknown
for all images, one cannot resolve projective ambiguity emerging in
reconstruction of SfM. It was then shown that if skew is zero or if at
least one of the internal parameters is unknown but constant among
images, projective ambiguity can be resolved and self-calibration is
feasible. However, these results only guarantee that self-calibration
is feasible if camera motion is general. When camera motion is a CMS,
the problem can become degenerate and self-calibration is no longer
feasible. As CMSs differ for different settings of self-calibration,
only a few settings which are practically important were well studied
\cite{sturmcritical1997,kahlcritical2000}, such as when focal lengths
are unknown and varying and all other parameters are known
\cite{sturmcritical2002}.
In this paper, we consider the case where skew and aspect ratio are
unknown and varying and all others are known. This setting was not
considered previously probably because of lack of its usefulness, and
ours is the first attempt in the literature of self-calibration.



\section{Modeling rolling shutter distortion}

\subsection{Constant motion model}

Let $\vec{X}\equiv[X,Y,Z]^\top$ and $\vec{x}\equiv[x,y,z]^\top$ denote
the world and camera coordinates, respectively. The coordinate
transformation between these two is given by
\begin{equation}
\vec{x}
=\mat{R}
(\vec{X}-\vec{p}),
\label{eqn:euclid}
\end{equation}
where $\mat{R}$ is a rotation matrix and $\vec{p}$ is a 3-vector (the
world coordinates of camera position). Assuming constant camera motion
during frame capture, rolling shutter (RS) distortion is modeled by
\begin{equation}
\begin{bmatrix}
c\\
r\\
1
\end{bmatrix}
\propto \mat{R}(r\gvec{\phi}) \mat{R}\{\vec{X}-(\vec{p}+r\vec{v})\},
\label{eqn:fullmodel0}
\end{equation}
where $c$ and $r$ are column ($x$) and row ($y$) coordinates,
respectively; the shutter is closed in the ascending order of $r$;
$\mat{R}(r\gvec{\phi})$ and $\vec{v}$ represent the rotation matrix and
translation vector of the camera motion;
$\gvec{\phi}=[\phi_1,\phi_2,\phi_3]^\top$ is the axis and angle
representation of the rotation. Note that $\mat{R}$ and $\vec{p}$ are
the camera pose when the shutter closes at $r=0$.

Assuming the angle of $\gvec{\phi}$ to be small, we approximate
\Eq{fullmodel0} as follows:
\begin{equation}
\begin{bmatrix}
c\\
r\\
1
\end{bmatrix}
\propto (\mat{I}+r[\gvec{\phi}]_\times) \mat{R}\{\vec{X}-(\vec{p}+r\vec{v})\},
\label{eqn:fullmodel}
\end{equation}
where
\begin{equation}
[\gvec{\phi}]_\times=
\begin{bmatrix}
0 & -\phi_3 & \phi_2\\
\phi_3 & 0 & -\phi_1\\
-\phi_2 & \phi_1 & 0
\end{bmatrix}.
\end{equation}
This model slightly differs from previous studies (e.g.,
\cite{alblr6p-rolling2015}) in the representation of the translational
component of the camera motion during frame capture. Our model
parameterizes the position of the camera in space independently of its
rotation, which is physically more natural.


\subsection{Rotation-only motion model}

In what follows we consider the case where $\vec{v}$ is small and
negligible. 
This is generally a good approximation unless a camera is close to
objects in a scene. Setting $\vec{v}=\vec{0}$ and using
\Eq{euclid} with $\vec{x}\equiv[x,y,z]^\top$, \Eq{fullmodel} can be
rewritten as follows:
\begin{equation}
\begin{bmatrix}
c\\
r\\
1
\end{bmatrix}
\propto (\mat{I}+r[\gvec{\phi}]_\times) 
\begin{bmatrix}
x\\
y\\
z
\end{bmatrix}
\propto (\mat{I}+r[\gvec{\phi}]_\times) 
\begin{bmatrix}
x'\\
y'\\
1
\end{bmatrix},
\label{eqn:distantcase}
\end{equation}
where $x'\equiv x/z$ and $y'\equiv y/z$. 

\subsection{A two-step model based on affine camera approximation}
\label{sec:affine}

\Eq{distantcase} models the RS distortion induced by rotating a
perspective camera with a constant velocity. Now, we approximate the
camera with an affine camera as
\begin{equation}
\label{eqn:dc2}
\begin{bmatrix}
c\\
r\\
1
\end{bmatrix}
\propto
(\mat{I}+r[\gvec{\phi}]_\times) 
\begin{bmatrix}
x'\\
y'\\
1
\end{bmatrix}
\approx
\begin{bmatrix}
  x' - r\phi_3 y' + r\phi_2 \\
  r\phi_3 x' + y' - r\phi_1 \\
 1
\end{bmatrix}.
\end{equation}
This model captures the ``principal components'' of the RS distortion,
as the approximation will be accurate when $x'$, $y'$, and $r$ are
small. This is true for image points around the principal points in
the case of general cameras or for any points in the case of cameras
with a telephoto lens.  Note, however, that we will {\em not} use this
affine camera approximation for the projection model used in SfM;
regarding the projection from $(c,r)$ to image coordinates $(u,v)$, we
will use a standard perspective camera model, as in \Eq{proj0}. Thus,
in the above approximation we simply neglect the effects of the second
order terms of the image coordinates within the model of RS
distortion. (Another proof without using this affine camera assumption
is given in the supplementary material.) The validity of this
approximated model will be shown experimentally in Section
\ref{sec:exp}.


\Eq{dc2} gives a transformation $f:[x',y']\mapsto [c,r]$. We show that
$f$ can be approximated by a composition of two transformations.


\begin{prop}
When $\phi_1$, $\phi_2$, and $\phi_3$ are small, 
$f:[x',y']\mapsto [c,r]$ can be approximated as
\begin{equation}
f\approx f_p \circ f_d,
\label{eqn:integ}
\end{equation}
where
$f_d:[x',y']\mapsto[x'',y'']$ is defined as
\begin{subequations}
  \label{eqn:distrot}
  \begin{align}
    x''&=
    x' - \phi_3 y'^2,\\
    y''&=
    y' + \phi_3 x'y',
  \end{align}
\end{subequations}
and $f_p:[x'',y'']\mapsto[c,r]$ as
\begin{equation}
  \begin{bmatrix}
    c\\r\\1
  \end{bmatrix}  
  \propto
  \begin{bmatrix}
    1 & \phi_2 & 0 \\
    0 & 1-\phi_1 & 0 \\
    0 & 0 & 1
  \end{bmatrix}
  \begin{bmatrix}
    x''\\y''\\1
  \end{bmatrix}.
  \label{eqn:distskewasp}
\end{equation}
\end{prop}


\begin{proof}
We will show that the approximation (\ref{eqn:integ}) holds true, that
is, its right hand and left hand sides are equal, in the first-order
approximation with respect to $\phi_1$, $\phi_2$, and $\phi_3$. As for
the right hand side, Substituting Eqs.(\ref{eqn:distrot}) into
\Eq{distskewasp} and neglecting the second order terms of $\phi_1$,
$\phi_2$, and $\phi_3$, we have
\begin{subequations}
\label{eqn:cr}
\begin{align}
c&= x'  + \phi_2 y' - \phi_3 y'^2,\\
r&= y'  - \phi_1 y' + \phi_3 x'y'.
\end{align}
\end{subequations}
As for the left hand side, substituting the second component of
\Eq{dc2}, $r= r\phi_3 x' + y' - r\phi_1$, into the first component and
then neglecting the second order terms also gives
(\ref{eqn:cr}a). Conducting similar substitution of $r$ into $r$
itself and neglecting the second order terms yields (\ref{eqn:cr}b).
\end{proof}


\section{Self-calibrating rolling-shutter cameras}

We employ Eqs.(\ref{eqn:integ})-(\ref{eqn:distskewasp}) as our RS
camera model. We first show that SfM based on the model is formulated
as a self-calibration problem.

\subsection{Problem formulation}

\begin{prop}
Suppose SfM from images captured by a camera with known internal
parameters. Assume the RS camera model given by
Eqs.(\ref{eqn:integ})-(\ref{eqn:distskewasp}) with unknown motion
parameters for each image. Then, the SfM problem is equivalent to
self-calibration of an imaginary camera that has unknown, varying skew
and aspect ratio along with varying lens distortion of a special kind.
\end{prop}

\begin{proof}
The first component $f_d$ of our two-step model $f=f_p\circ f_d$ may
be interpreted as a kind of lens distortion. The most widely used
model of lens distortion would be
\begin{subequations}
\small
\label{eqn:lensdist}
\begin{align}
x''&=x'+x'(k_1\lambda^2+k_2\lambda^4)+2p_1x'y'+p_2(\lambda^2+2x'^2)\\
y''&=y'+y'(k_1\lambda^2+k_2\lambda^4)+p_1(\lambda^2+2y'^2)+2p_2x'y',
\end{align}
\end{subequations}
where $x'\equiv x/z$ and $y'\equiv y/z$; $\lambda^2\equiv
x'^2+y'^2$. The second terms on the right hand sides represent radial
distortion and the third terms represent tangential distortion. The
transformation $f_d$ of \Eq{distrot} has a similar form; there are
second-order additive terms $y'^2$ and $x'y'$, which appear in the
tangential distortion terms.

Consider the second component $f_p$ of \Eq{distskewasp}. The row and
column coordinate $[c,r]$ are transformed to image coordinates $[u,v]$
by the matrix $\mat{K}$ of internal parameters as
\begin{equation}
\label{eqn:proj0}
\begin{bmatrix}
u\\v\\1
\end{bmatrix}\propto
\mat{K}
\begin{bmatrix}
c\\r\\1
\end{bmatrix}
\propto
\begin{bmatrix}
f & 0 & u_0\\
0 & f & v_0\\
0 & 0 & 1
\end{bmatrix}
\begin{bmatrix}
c\\r\\1
\end{bmatrix},
\end{equation}
where $f$ and $(u_0,v_0)$ are the focal length and the principal point
of this camera. We assume here zero skew and unit aspect ratio for the
sake of simplicity. The choice of other values will not change the
discussions below. The substitution of \Eq{distskewasp} into the above
yields%
{\small 
\begin{equation}
\begin{bmatrix}
u\\
v\\
1
\end{bmatrix}
\propto
\begin{bmatrix}
f & 0 & u_0\\
0 & f & v_0\\
0 & 0 & 1
\end{bmatrix}
\begin{bmatrix}
1 & \phi_2 & 0\\
0 & 1-\phi_1 & 0\\
0 & 0 & 1
\end{bmatrix}
\begin{bmatrix}
  x''\\
  y''\\
  1
\end{bmatrix}\nonumber\\
=
\begin{bmatrix}
f & \phi_2 f & u_0\\
0 & (1-\phi_1) f & v_0\\
0 & 0 & 1
\end{bmatrix}
\begin{bmatrix}
  x''\\
  y''\\
  1
\end{bmatrix}.
\label{eqn:skewasp}
\end{equation}}
The matrix gives an internal parameter matrix of the imaginary camera;
$\phi_1$ (rigorously $1-\phi_1$) is its aspect ratio, and $\phi_2$ is
its skew.
\end{proof}



\subsection{Theory of self-calibration}


Generally, self-calibration is to estimate (partially) unknown internal
parameters of cameras from only images of a scene and thereby obtain
metric 3D reconstruction of the scene. It was extensively studied from
1990's to early 2000's
\cite{maybanka1992,heydeneuclidean1997,pollefeysself-calibration1999,heydenminimal1998,heydenflexible1999,hartleymultiple2003}.
In what follows, while summarizing these studies, we will describe how
our problem can be dealt with.

\subsubsection{Feasibility}

If the internal parameters of cameras are all unknown, we can obtain
3D reconstruction up to projective ambiguity. It was studied
previously what kind of knowledge about which internal parameters is
necessary to resolve projective ambiguity. It was shown
\cite{pollefeysself-calibration1999,heydenminimal1998} that it is
sufficient if skew vanishes for all the images. It was also shown
\cite{heydenflexible1999} that it suffices that only one of the five
internal parameters is constant (but unknown) for all the images;
others may be unknown and varying for different images. This result
proves that we can indeed perform self-calibration formulated above.


Thus, theoretically, one can conduct self-calibration only if she/he
has a little knowledge about internal parameters. Practically,
however, this is not true because of two reasons. One is the existence
of critical motion sequences (CMSs), which will be discussed
later. The other is that when a large portion of internal parameters
are unknown, it becomes hard to determine them with practical
accuracy. Considering practical requirements, it is the most popular
to assume only focal lengths \cite{pollefeysself-calibration1999} (and
sometimes additionally principal points \cite{kahlcritical2000}) to be
unknown and varying, and all other parameters to be known. It is
interesting to point out that our case is the complete reversal in the
choice of known/unknown parameters.


Additionally, simple counting argument gives a necessary condition for
self-calibration. Let $m$ be the number of cameras (viewpoints), $n_k$
be the number of known internal parameters, and $n_f$ be the number of
constant (but unknown) internal parameters. Then, a necessary
condition is given \cite{hartleymultiple2003} by
\begin{equation}
mn_k + (m-1)n_f \geq 8.
\end{equation}
In our case, $n_k=3$ and $n_f=0$, from which we have $m\geq 3$ as a
necessary condition.

\subsubsection{Critical motion sequences}

Even in cases where the above theories justify the feasibility of
self-calibration, there could emerge degeneracy depending on camera
poses. A set of camera poses for which the parameters cannot be
uniquely determined is called a critical motion sequence (CMS).

CMSs were studied in parallel with the above studies of the feasibility
of self-calibration. A catalog of CMSs for the case of constant
camera parameters is shown in \cite{sturmcritical1997}. There are also
studies for cases where some parameters are known and/or some vary. In
\cite{kahlcritical2000}, each of the cases where skew vanishes,
additionally aspect ratio is known, and further additionally principal
points are known are studied. In
\cite{sturmcritical2002},
a catalog of CMSs is shown for the case where focal length is
unknown and varying, and others are known; this is the most popular
setting of self-calibration.

However, the case considered in this paper, i.e., when skew and aspect
ratio are both unknown and varying whereas all others are known, has not
been considered in the literature. This may be due to lack of
applications of the setting. Thus, we will examine CMSs for this case
in Sec.\ref{sec:CMSs}.






\subsubsection{Lens distortion}

When conducting self-calibration of cameras (or simply bundle
adjustment), it is common to jointly estimate lens distortion (i.e.,
$k_1,\ldots$ and $p_1,\ldots$ in \Eq{lensdist}) together with other
parameters. Researchers have traditionally treated the lens distortion
parameters separately from the five internal camera parameters. In
fact, their uniqueness and existence have not been rigorously
discussed until very recently \cite{wucritical2014}. This treatment is
indeed reasonable considering nonlinear nature of the lens distortion;
projective transformation maps a line onto a line, whereas lens
distortion does not. In this study, adopting the same position and
assuming that the parameter $\phi_3$ of $f_d$ can be treated equally
to lens distortion parameters, we assume that $\phi_3$ can be
determined uniquely independently of other parameters.



\subsection{CMSs for RS camera self-calibration}
\label{sec:CMSs}

We now consider CMSs for our self-calibration problem. We begin with
derivation of general-purpose equations following Sec.19.2 of
\cite{hartleymultiple2003}, and then tailor them for our case.


\subsubsection{Equations of self-calibration}

Suppose that a projective reconstruction $\{\mat{P}^i,\vec{X}_j\}$ is
given, where $\mat{P}^i$ is a projection matrix and $\vec{X}_j$ is a
scene point. Let $\mat{P}^i=[\mat{A}^i\;\vert\; \vec{a}^i]$. By
choosing a coordinate system we may assume that $\mat{A}^1=\mat{I}$
and $\vec{a}^1=\vec{0}$. Let $\mat{K}^i$ be the internal parameter
matrix of camera $i$. The DIAC (dual image of the absolute conic) for
image $i$ is given by $\omega^{*i}=\mat{K}^i\mat{K}^{i\top}$. A
three-dimensional projective transformation $\mat{H}$ that transforms
the projective reconstruction $\{\mat{P}^i,\vec{X}_j\}$ into a metric
reconstruction $\{\mat{P}^i\mat{H},\mat{H}^{-1}\vec{X}_j\}$ can be
encoded as
\begin{equation}
\mat{H}=\begin{bmatrix}
\mat{K}^1 & \vec{0}\\
-\vec{p}^\top\mat{K} & 1
\end{bmatrix}.
\end{equation}
Then, equations of self-calibration $(i=1,\ldots)$ are given by
\begin{equation}
  \omega^{*i}=\left(\mat{A}^i-\vec{a}^i\vec{p}^\top\right)
  \omega^{*1}
\left(\mat{A}^i-\vec{a}^i\vec{p}^\top\right)^\top.
\label{eqn:omega}
\end{equation}
Here, $\omega^{*i}$ $(i=1,\ldots)$ and $\vec{p}$ are unknowns.  If
there is knowledge about some internal parameters, it gives
constraint(s) on the elements of $\omega^{*i}$. 

We now derive such constraints in our case. As the principal point is
known for any image, we can transform image coordinates so that it
will be $(0,0)$. Then the internal parameter matrix will be 
\begin{equation}
\mat{K}^i=\begin{bmatrix}
f^i & s^if^i & 0\\
0   & \alpha^if^i & 0\\
0   & 0           & 1
\end{bmatrix}.
\label{eqn:ourK}
\end{equation}
The DIAC will be
\begin{equation}
\omega^{*i}=\mat{K}^i\mat{K}^{i\top}=\begin{bmatrix}
f^2 + s^2f^2 & s \alpha f^2 & 0\\
s \alpha f^2 & \alpha^2 f^2 & 0\\
0 & 0 & 1
\end{bmatrix},
\end{equation}
where superscripts ($i$'s) are omitted in the matrix. From this we
have constraints on the elements of the DIACs. The elimination of $s$
and $\alpha$ from the upper-left $2\times 2$ block yields
\begin{subequations}
\label{eqn:selfeqs}
\begin{equation}
  \omega^{*i}_{11}\omega^{*i}_{22}=\omega^{*i}_{22}f^2+\omega^{*i2}_{12},
\end{equation}
and the $(1,3)$ and $(2,3)$ elements vanish as
\begin{equation}
  \omega^{*i}_{13}=\omega^{*i}_{23} = 0.
\end{equation}
\end{subequations}
Substituting \Eq{omega} into the above equations to eliminate the
elements of $\omega^{*i}$ ($i\neq 1$), we have equations that have
$\omega^{*1}$ and $\vec{p}$ as only unknowns. Solving them for these
unknowns, we can determine $\mat{H}$, from which we can obtain a
metric reconstruction. 

\subsubsection{Representation of CMSs}
\label{sec:cmsrep}

A CMS is a set of camera poses for which the above equations becomes
degenerate. Specifically, there are three equations for each camera
$i$, which are obtained by substituting \Eq{omega} into
Eqs.(\ref{eqn:selfeqs}) as described above. They differ for different
cameras owing to the difference in $\mat{A}_i$ and
$\vec{a}_i$. However, they can be degenerate if there is some special
relation in $\{\mat{A}_i,\vec{a}_i\}$ among cameras.

This analysis of degeneracy theoretically gives a complete set of
CMSs. Indeed, one can check whether any given camera motion is a CMS
or not. However, it is only implicitly represented, which makes
intuitive understanding hard. This is also the case with other, more
popular self-calibration settings such as the case of constant, unknown
internal parameters \cite{sturmcritical1997} and the case of zero skew
and known aspect ratio \cite{kahlcritical2000}. For such cases,
researchers have derived several practically important
cases. Similarly, we derive an intuitive CMS here, which coincides
with the one given in \cite{Albleccv2016}. Thus, the following gives
yet another explanation of the CMS.

\begin{prop}
All images are captured by cameras having the parallel $y$ axis. This
camera motion is a CMS. The translational components may be arbitrary.
\label{prop3}
\end{prop}

\begin{proof}
We denote rotation about the $y$ axis with angle $\varphi$ by
$\mat{R}_y(\varphi)$. Choosing the world coordinate system, we may
write the image projection of a scene point $[X,Y,Z]^\top$ as
\noindent
\begin{equation}
\small
\mat{K}
\left(\mat{R}_y(\varphi) \begin{bmatrix}
X\\Y\\Z
\end{bmatrix}
+ \begin{bmatrix}
t_X\\t_Y\\t_Z
\end{bmatrix}
\right)
=
\begin{bmatrix}
f(X\cos\varphi +Z\sin\varphi +t_X) + f s(Y+t_Y)\\
f \alpha (Y+t_Y)\\
-X\sin\varphi  + Z\cos\varphi +t_Z
\end{bmatrix},
\end{equation}%
where $\mat{K}$ is parameterized as in \Eq{ourK}. It is observed that
the image will not change if we multiply the $Y$ coordinate by $k$ and
both $s$ and $\alpha$ by $1/k$, as $s(Y+t_Y)=sk^{-1}\cdot k(Y+t_Y)$
and $\alpha(Y+t_Y)=\alpha k^{-1}\cdot k(Y+t_Y)$. This means that we
cannot determine the scale $k$ from only these images, and this camera
motion is a CMS.
\end{proof}

An example of this CMS is the case where one acquires images while
holding a camera horizontally with zero elevation angle. This style of
capturing images is quite common. Even if a camera motion does not
exactly match this CMS, if it is somewhat close to it, estimation
accuracy could deteriorate depending on how close it is; see
\cite{Albleccv2016} for more detailed discussions.

\subsubsection{A method for resolving the CMS}
\label{sec:cmsremedy}

Besides simply avoiding it, there are several possible ways to cope
with this CMS. We propose to use {\em some} method outside the
framework of self-calibration to determine either $\phi_1$ or $\phi_2$
for a {\em single} camera that is selected somehow. Specifying only a
single value resolves the ambiguous scale $k$ mentioned above,
resolving the ambiguity of solutions. A practically easy way to do
this is to find out an image $i$ in the sequence that undergoes no RS
distortion (or as small distortion as possible), for which we set
$\phi_1^{i}=0$ and/or $\phi_2^{i}=0$. We will show through experiments
that this is indeed effective, provided that there exists such a
distortion-free image in the sequence and it can be identified.
Besides finding a viewpoint with $\phi_1^{i}=0$ or $\phi_2^{i}=0$, we
will be able to employ various methods depending on external
information available. 



\section{Experimental results}
\label{sec:exp}


Our self-calibration-based formulation of RS SfM can be used not only
for deriving CMSs but also for performing RS SfM in practice. To
demonstrate how degeneracy arising in RS SfM can be treated by our
method and also to evaluate the validity of the employed approximation
of the RS model, we apply it to RS images and compare the results with
the ground truths and those obtained by standard BA with/without the
standard rotation-only linearized RS camera model etc.

\subsection{Incorporating the RS camera model to SfM}

The proposed RS model can be incorporated into bundle adjustment in
the following way. The projection of a scene point $\vec{X}$ when our
RS model is assumed is as follows: $\vec{X}$ is first transformed by
\Eq{euclid} into the camera coordinates $\vec{x}$, and then
transformed by Eqs.(\ref{eqn:distrot}) and \Eq{skewasp} in turn to
yield image coordinates $[u,v]$. We assume the (genuine) internal
parameters (i.e., $\mat{K}$ of \Eq{proj0}) of each camera to be
known. Thus, unknowns per each camera are the six parameters of camera
pose plus the three RS parameters $[\phi_1,\phi_2,\phi_3]$. Bundle
adjustment is performed over these parameters for all the cameras in
addition to the parameters of point cloud. We used ceres-solver
\cite{ceressolver} for implementing bundle adjustment in our
experiments. The initial values of the RS parameters are set to 0,
since their true values should be small due to the nature of RS
distortion. The initial values of camera poses and point cloud are
obtained by using these initial values for the RS parameters.





\subsection{Synthetic image experiments}





\begin{figure}[bt]
\hfil
\hspace*{-1.5mm}
\hbox{
\tfbox{\includegraphics[width=20mm]{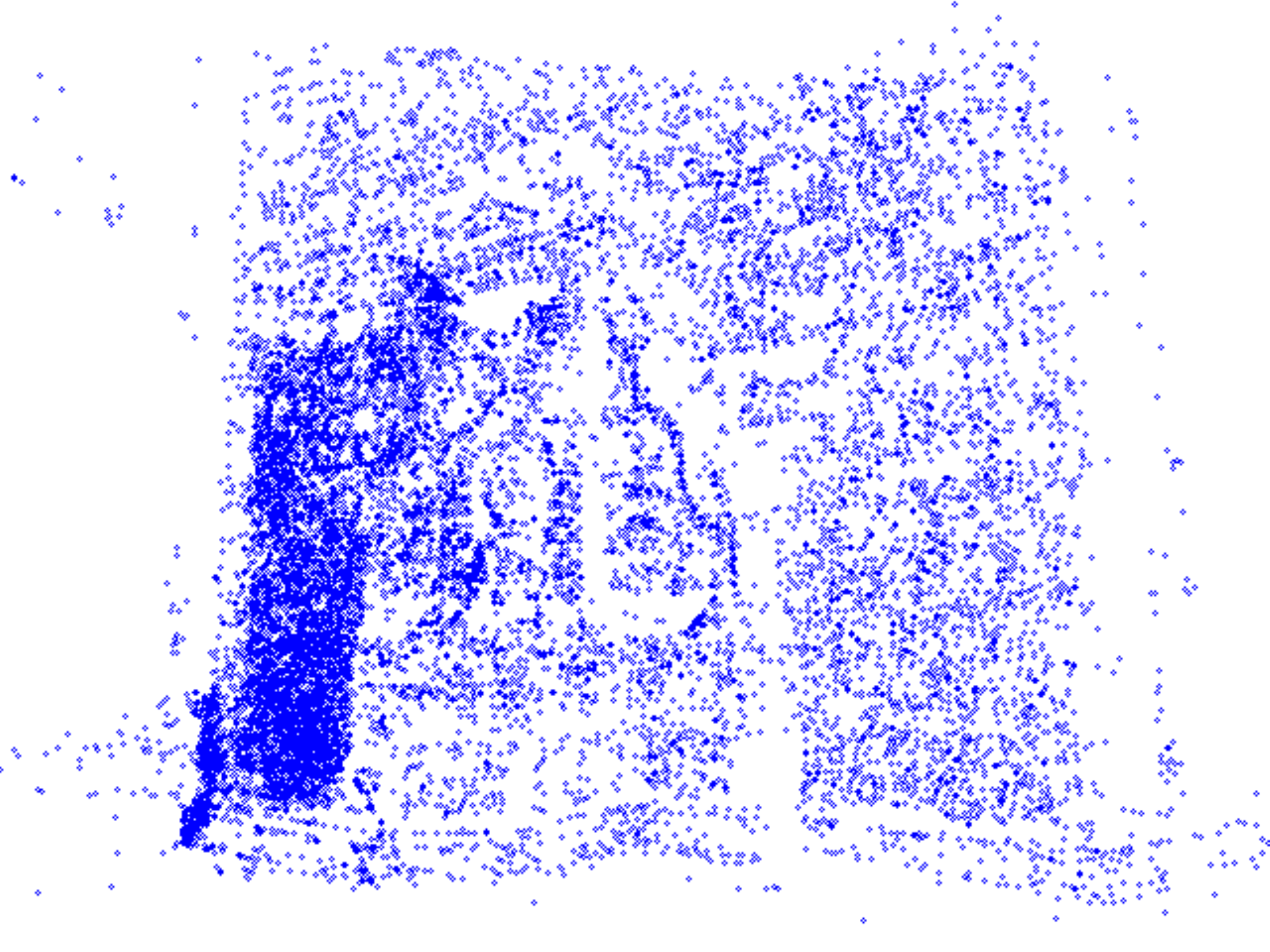}}
\tfbox{\includegraphics[width=20mm]{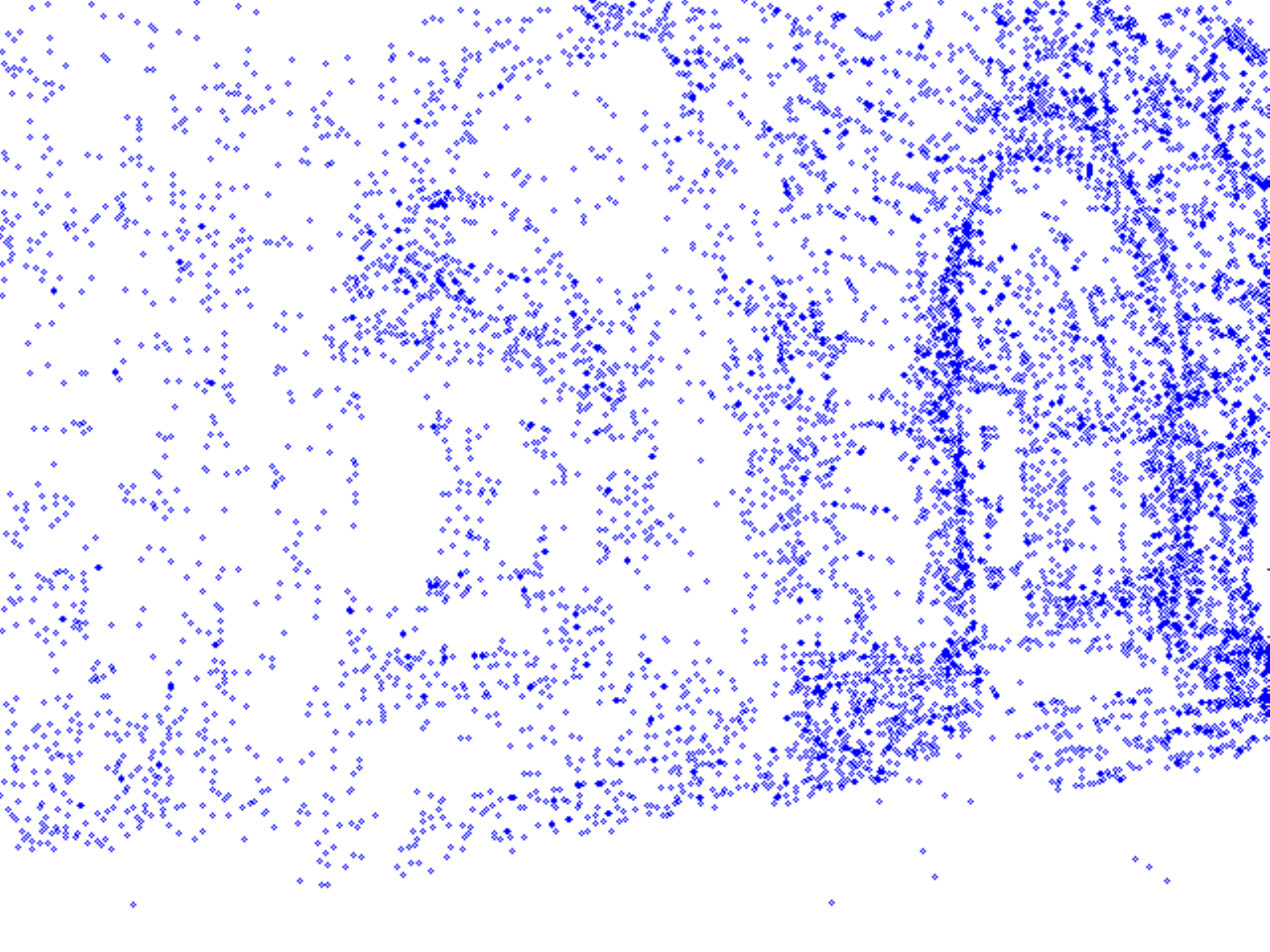}}
~\tfbox{\includegraphics[width=20mm]{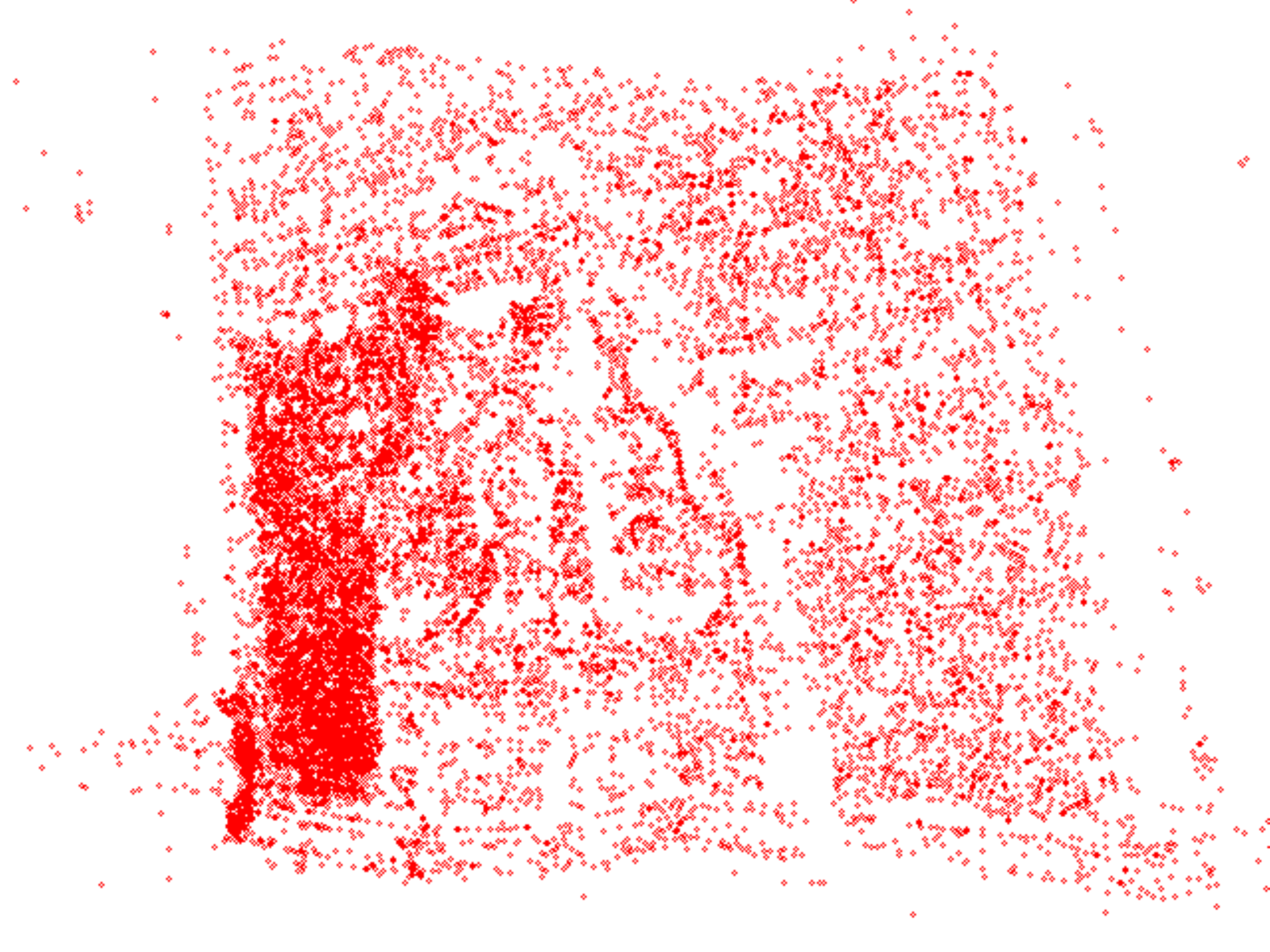}}
\tfbox{\includegraphics[width=20mm]{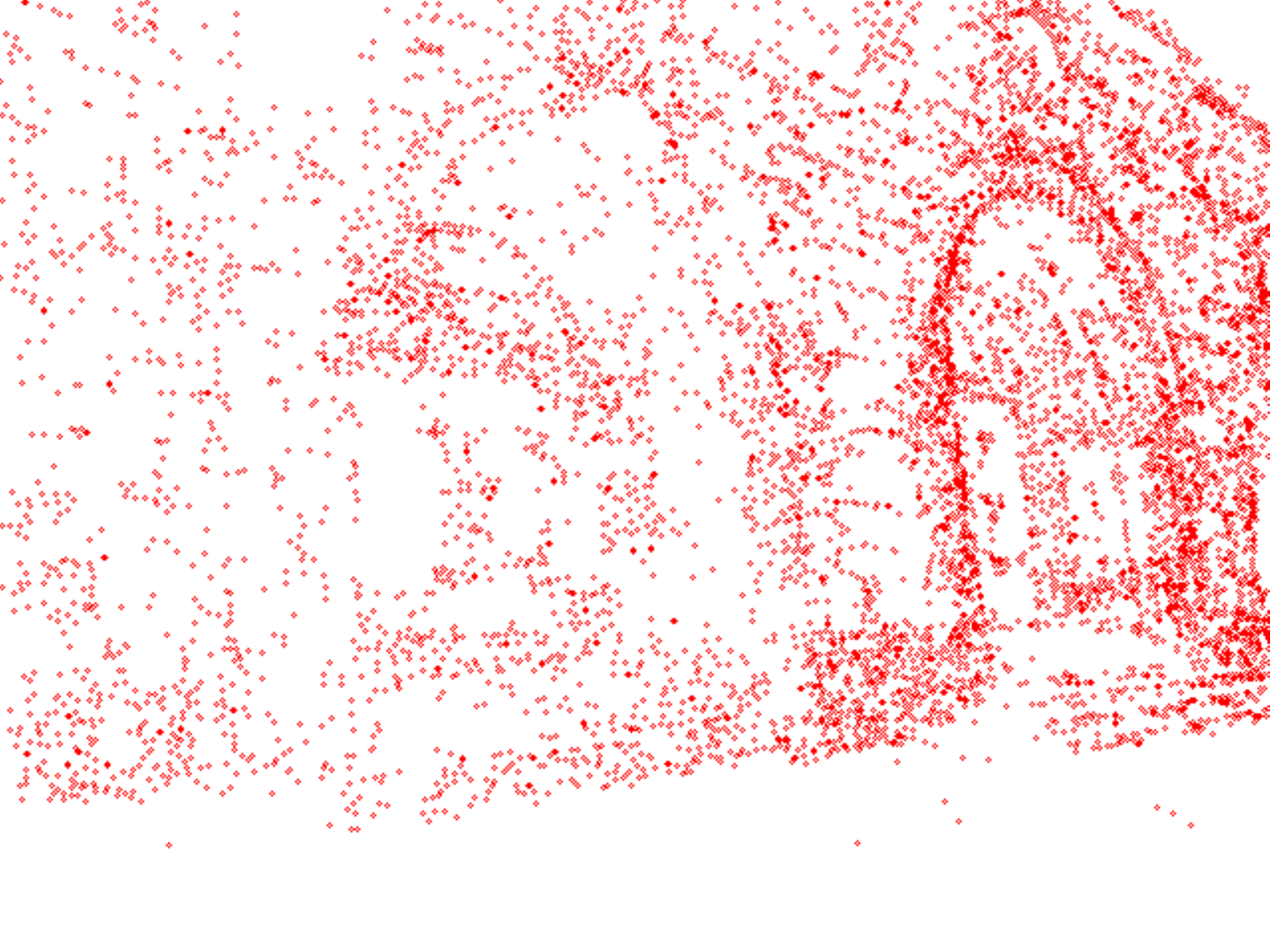}}}
\caption{Example of images of points used in Experiment I. Upper: Two
  original images of fountain-P11. Lower: With simulated RS
  distortion.}
\label{fig:rsexample}
\end{figure}

We first show results using synthetic images, for which we can
evaluate accuracy of 3D reconstructions using their ground truths.  In
the experiment, for given point cloud (scene points) and camera poses,
we synthetically generate images by projecting the points onto each
camera with the RS camera model (\ref{eqn:fullmodel0});
\Fig{rsexample} shows examples. RS distortion can generally affect
every step of the SfM pipeline, i.e., matching points among images,
computing initial values of camera poses as well as point cloud, and
finally bundle adjustment. As the proposed method is only concerned
with the last step, this experiment enables to evaluate its
effectiveness independently of others.


The details are as follows. To obtain point cloud and camera poses
that simulate natural scene structures and camera poses, we used
Visual SFM \cite{vsfm,wutowards2013} to reconstruct them from real
images. We used several image sequences listed in Table \ref{tbl:seqs}
from public datasets.
\begin{table}[t]
\small
\caption{Image sequences used for Experiment I and II.}
\label{tbl:seqs}
\centering
\begin{tabular}{l|ll}\hline
Sequence & Dataset & \# of images\\ \hline
fountain-P11 & EPFL\cite{strechaon2008}\footnote{http://cvlabwww.epfl.ch/data/multiview/denseMVS.html} & 11\\
Herz-Jesu-P8 & EPFL & 8\\
castle-P19 & EPFL & 19\\
Temple & Middleburry\footnote{http://vision.middleburry.edu/mview/data/} & 10(templeR0014-0023)\\\hline
\end{tabular}
\end{table}
We then projected the reconstructed point cloud onto the images of the
reconstructed cameras using the RS camera model
(\ref{eqn:fullmodel0}). The camera motion for each image was randomly
generated. For the rotation $R(r\gvec{\phi})$, the axis
$\gvec{\phi}/\lvert\gvec{\phi}\rvert$ was generated in a full sphere
with a uniform distribution and the angular velocity
$\lvert\gvec{\phi}\rvert$ was set so that
$(r_{max}-r_{min})\lvert\gvec{\phi}\rvert$ equals to a random variable
generated from a Gaussian distribution $N(0,\sigma_{rot}^2)$, where
$r_{max}$ and $r_{min}$ are the top and bottom rows of images. For the
translation, each of its three elements was generated according to
$N(0,(\sigma_{trans}\bar{t})^2)$, where $\bar{t}$ is the average
distance between consecutive camera positions in the sequence. We set
$\sigma_{rot}=0.05$ radians and $\sigma_{trans}=0.05$. We used the
same internal camera parameters as the original reconstruction. We
added Gaussian noises $\varepsilon, \varepsilon'\sim N(0,0.5^2)$ to
the $x$ and $y$ coordinates of each image point.

We conducted the following procedure for 100 trials for each of the
image sequences listed in Table \ref{tbl:seqs}. In each trial, we
regenerated the additive image noises and initial values for bundle
adjustment. RS distortion for each image (except the first image) of
each sequence was randomly generated once and fixed throughout the
trials. We intentionally gave {\em no distortion to the first
  image}. 

We applied four methods to the data thus generated. The first one is
ordinary bundle adjustment without any RS camera model, which is
referred to as ``w/o RS.'' The second and third ones are bundle
adjustment incorporating the proposed RS camera model. The second one
is to optimize all the RS parameters equally in bundle adjustment,
which is referred to as ``w/ RS.''  The third one is to set
$\phi_1^{1}=0$ and optimize all others, which is referred to as ``w/
RS*''. This implements the proposed approach of resolving the CMS
mentioned in Sec.\ref{sec:cmsremedy}. The last one is bundle
adjustment incorporating the exact RS model with linearized rotation
(\ref{eqn:distantcase}), which is referred to as ``w/
RS$(r[\phi]_\times])$''.

Figure \ref{fig:pointexp} shows the results. The plots show cumulative
error histograms of rotation and translation of cameras and of
structure (scene points). The rotation error is measured by the
average of the absolute angle of $\mat{R}^i\hat{\mat{R}^i}$ over the
viewpoints, where $\mat{R}^i$ and $\hat{\mat{R}^i}$ are the true and
estimated camera orientations, respectively. To eliminate scaling
ambiguity for the evaluation of translation and structure, we apply a
similarity transformation to them so that the camera trajectory 
be maximally close to the true one in the least-squares sense. Then
the translation error is measured by the average of differences
between the true camera position $\vec{p}^i$ and the estimated one
$\hat{\vec{p}^i}$ over the viewpoints. The structure error is measured
by the sum of distances between true points and their estimated
counterparts.
\begin{figure*}[t]
\tiny
\hspace*{-1.5mm}
\begin{tabular}{cccc}
\includegraphics[width=43.0mm]{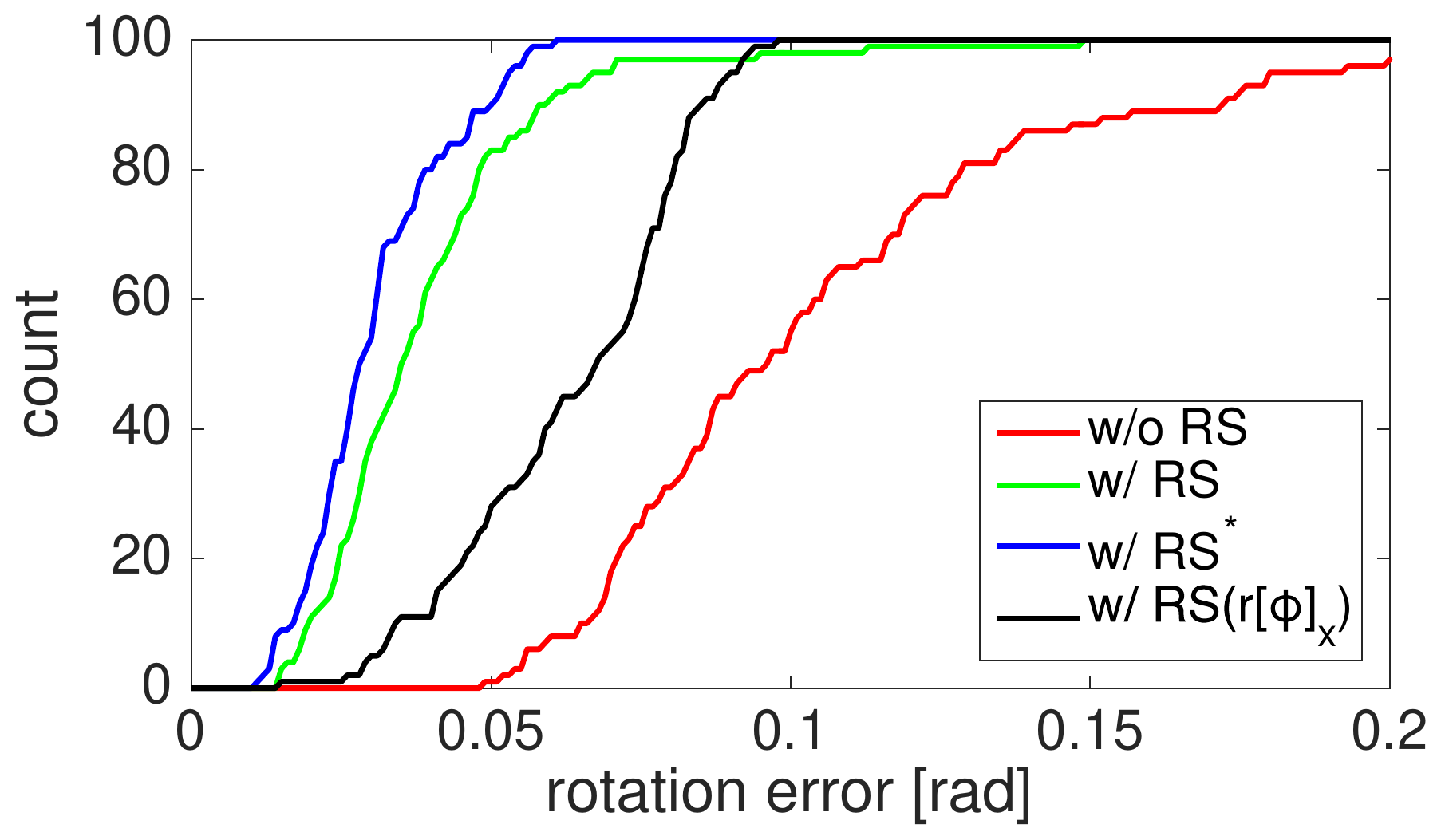}
\hspace*{-4mm}&\includegraphics[width=43.0mm]{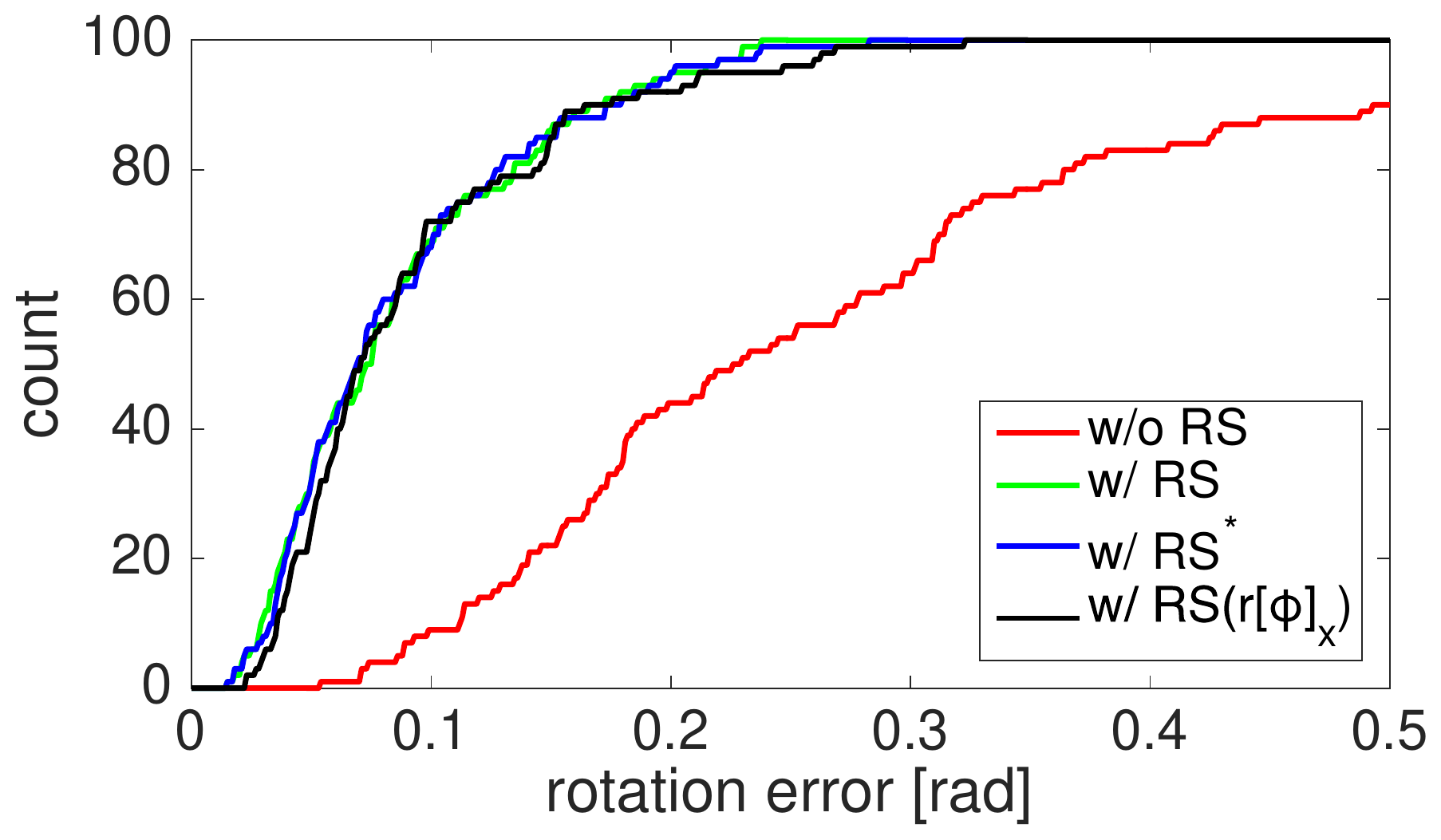}
\hspace*{-4mm}&\includegraphics[width=43.0mm]{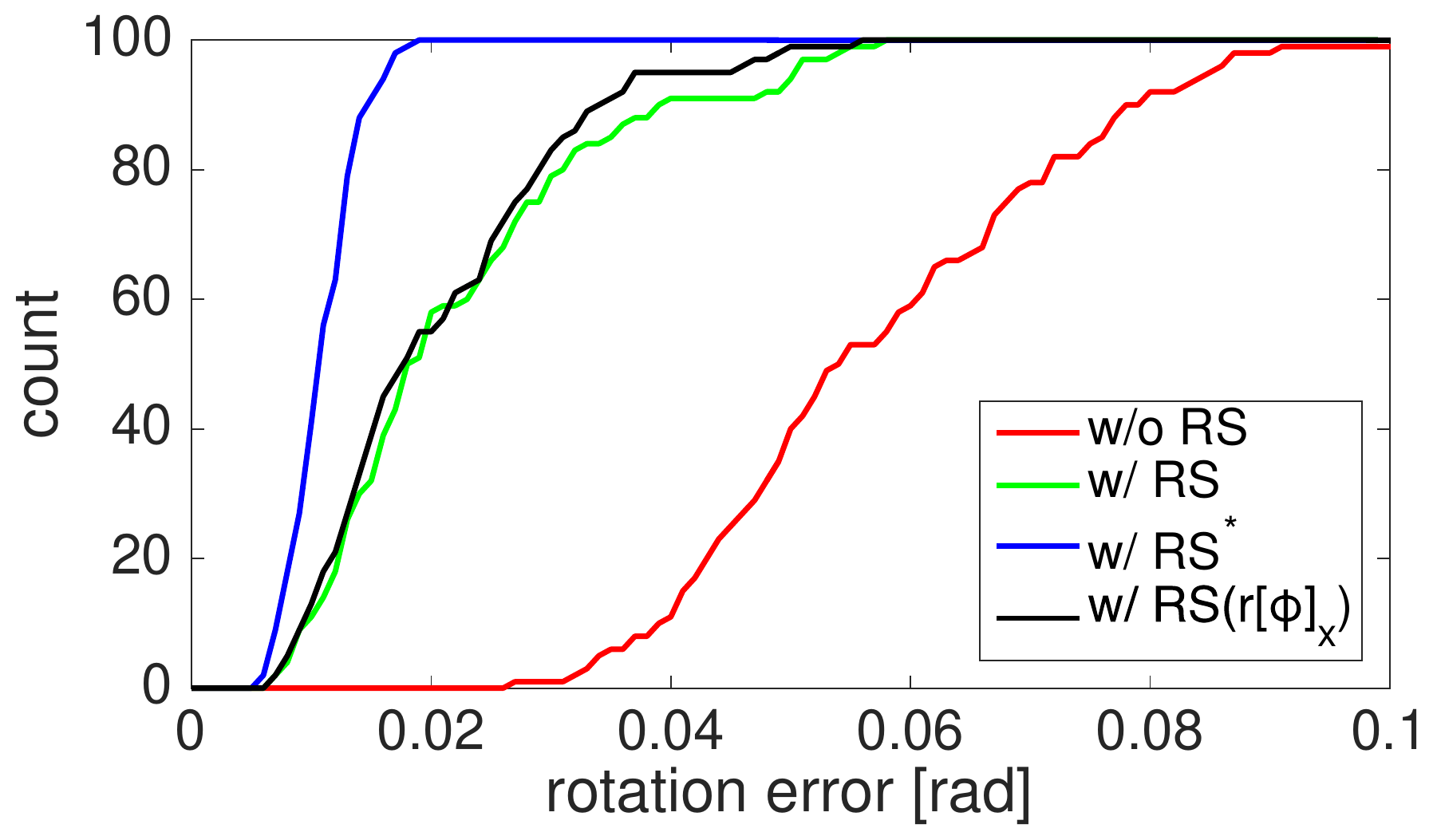}
\hspace*{-4mm}&\includegraphics[width=43.0mm]{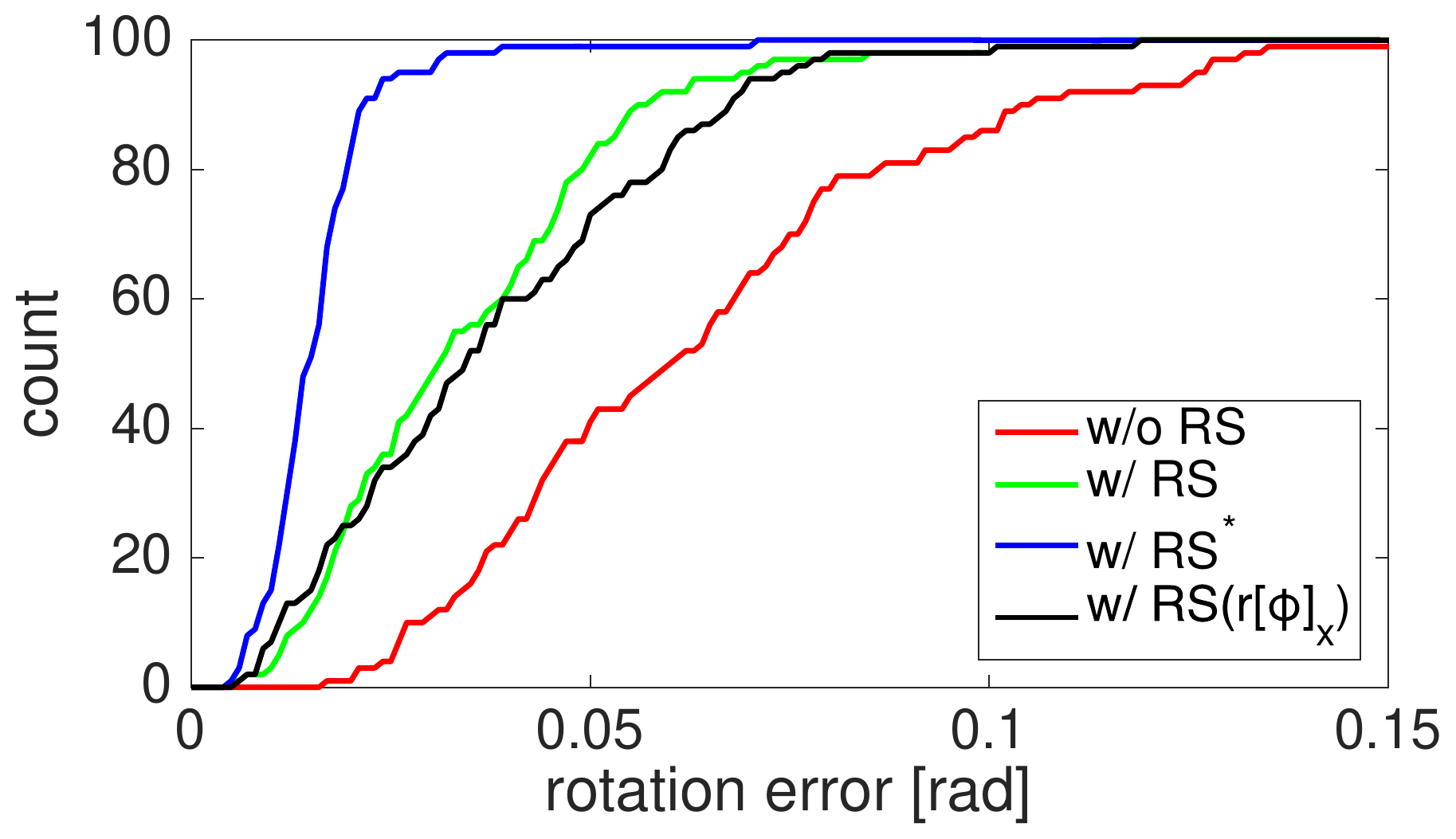}\\

\includegraphics[width=43.0mm]{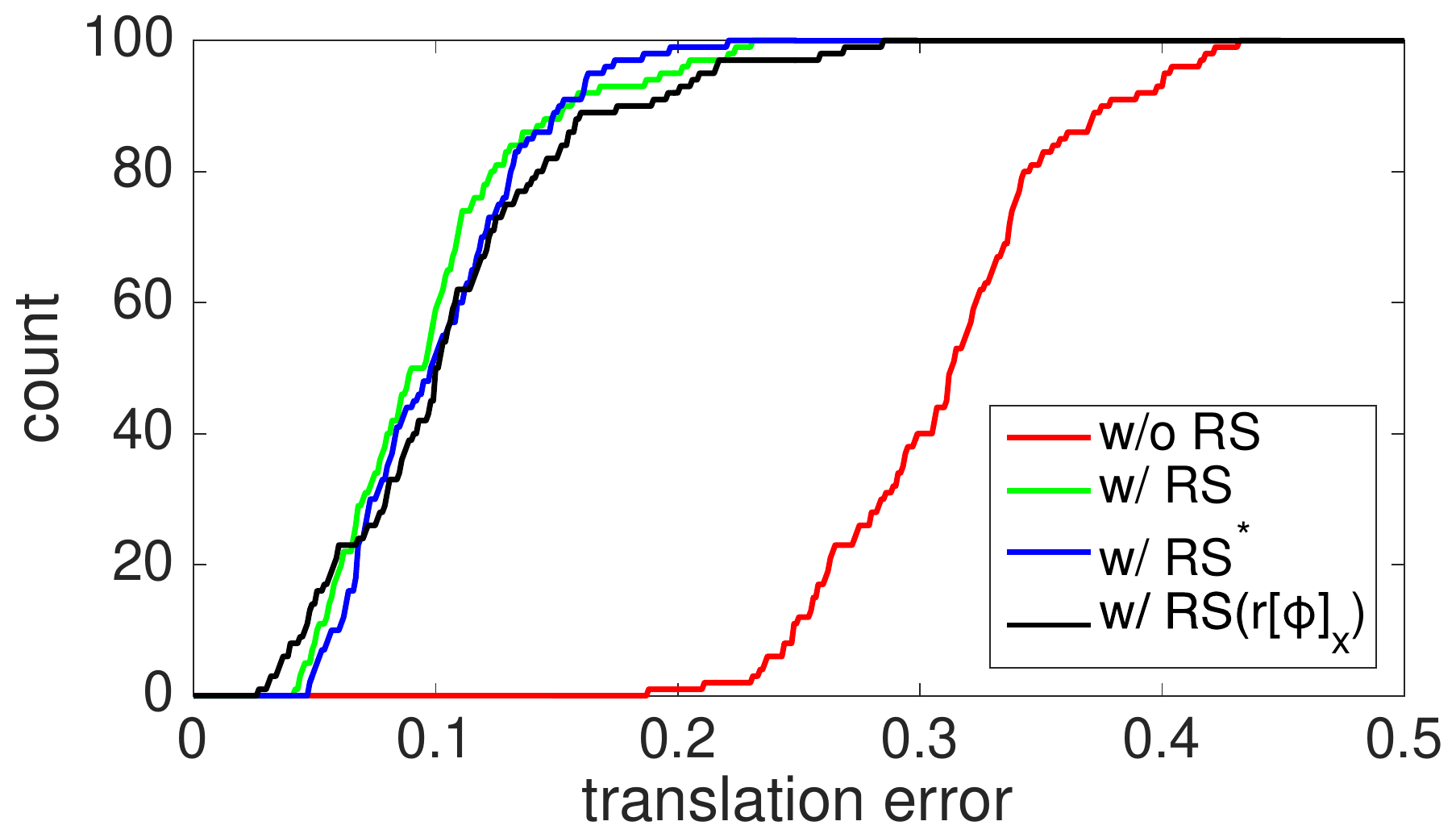}
\hspace*{-4mm}&\includegraphics[width=43.0mm]{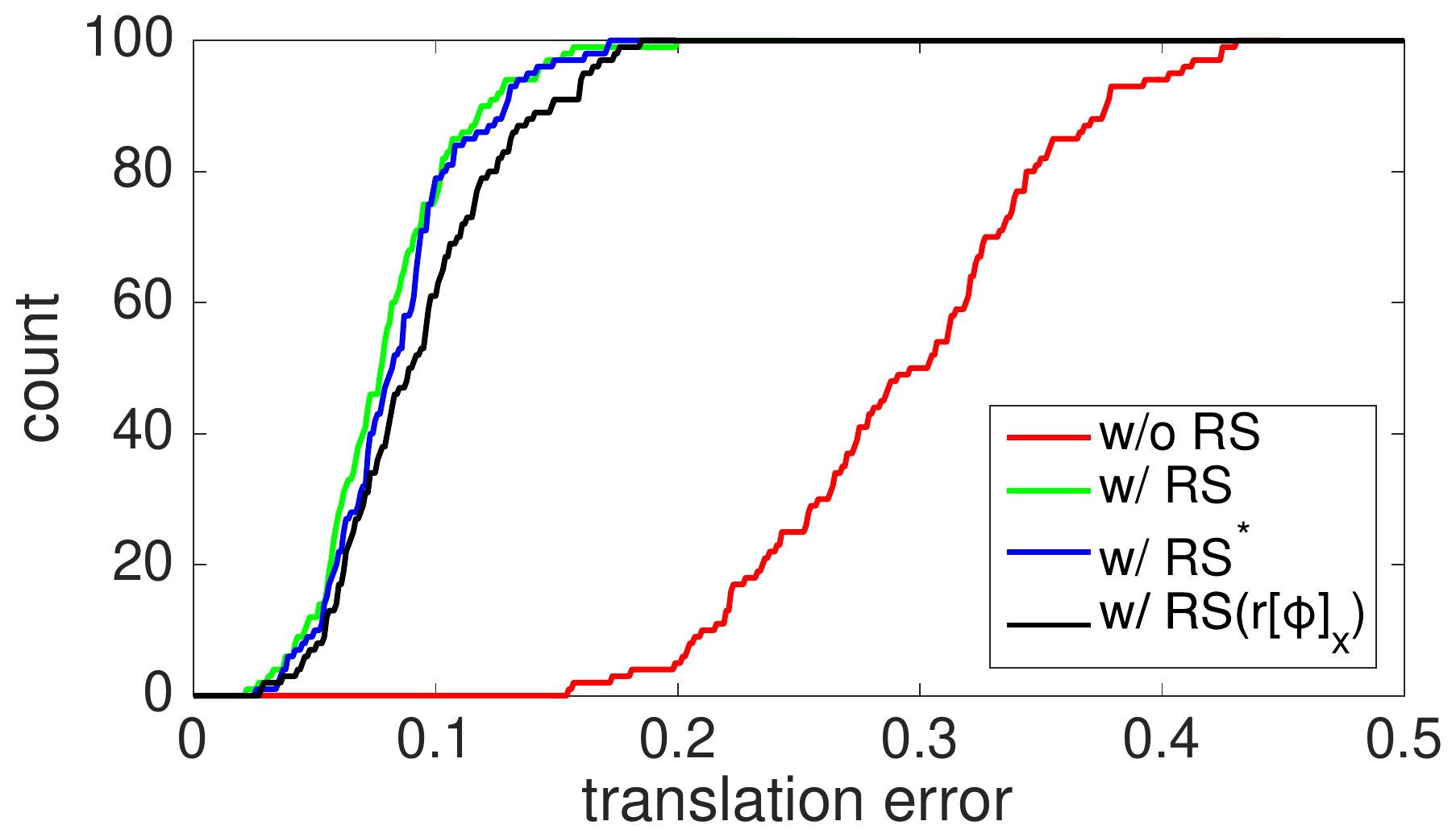}
\hspace*{-4mm}&\includegraphics[width=43.0mm]{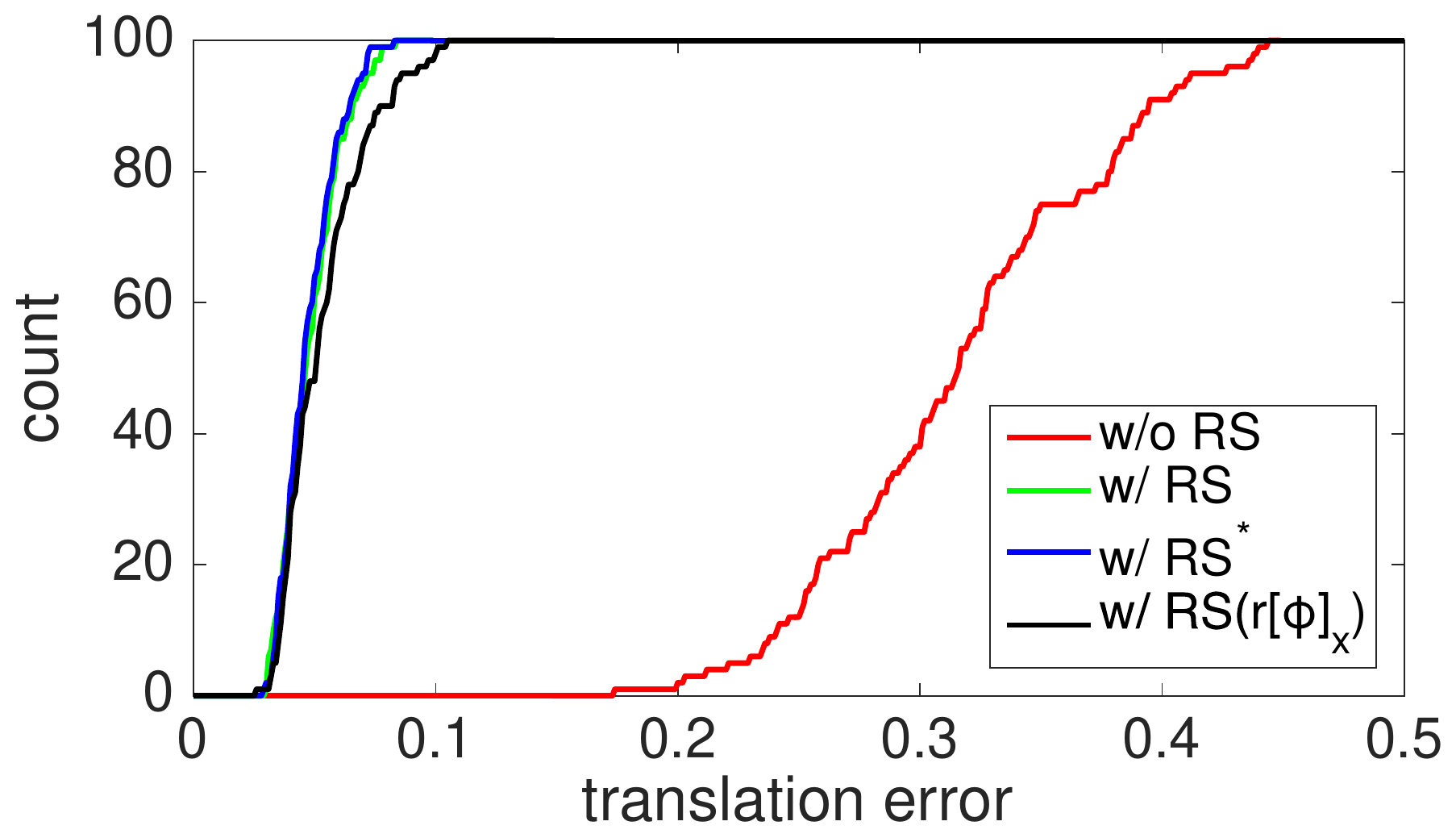}
\hspace*{-4mm}&\includegraphics[width=43.0mm]{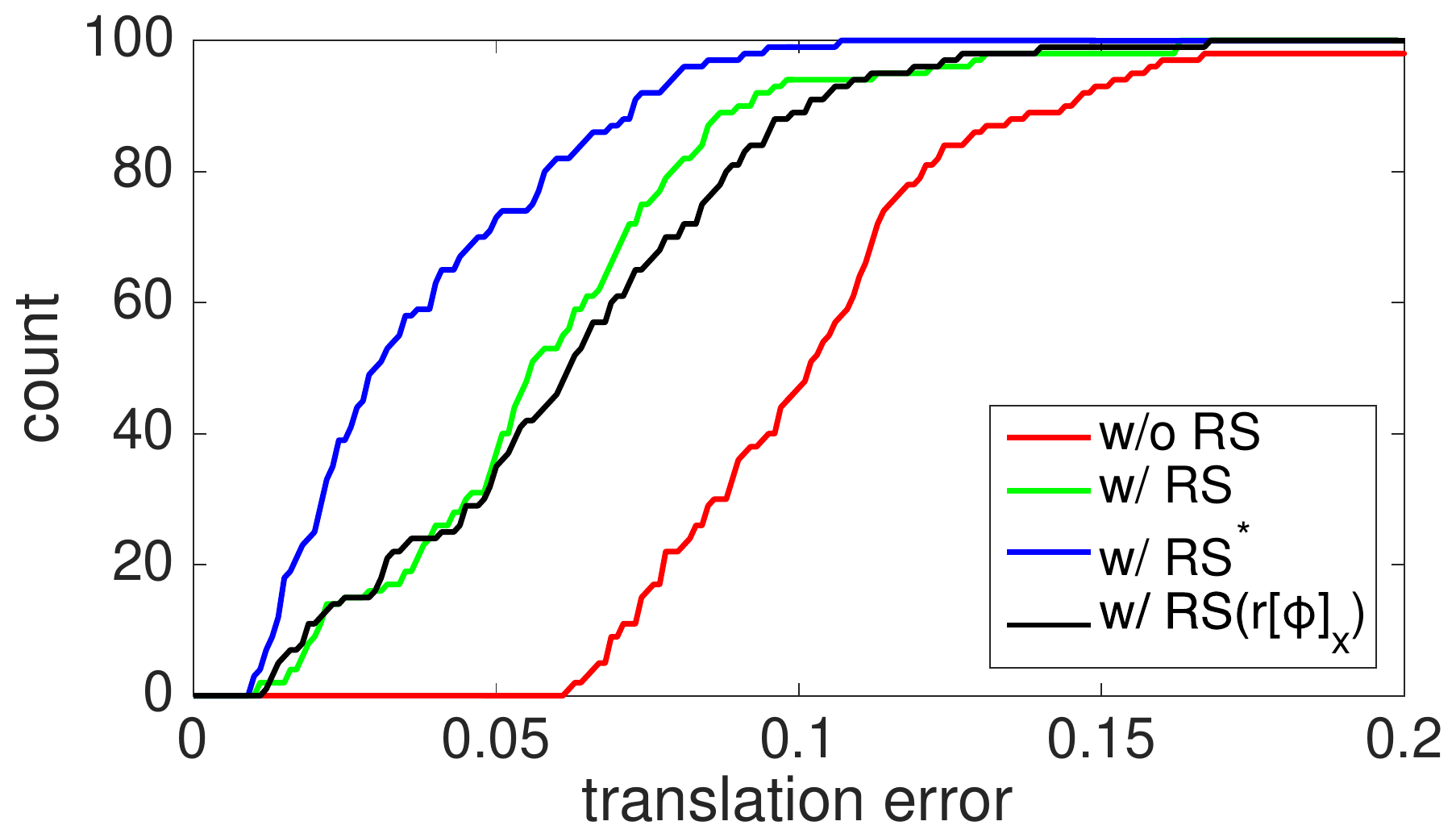}\\

\includegraphics[width=43.0mm]{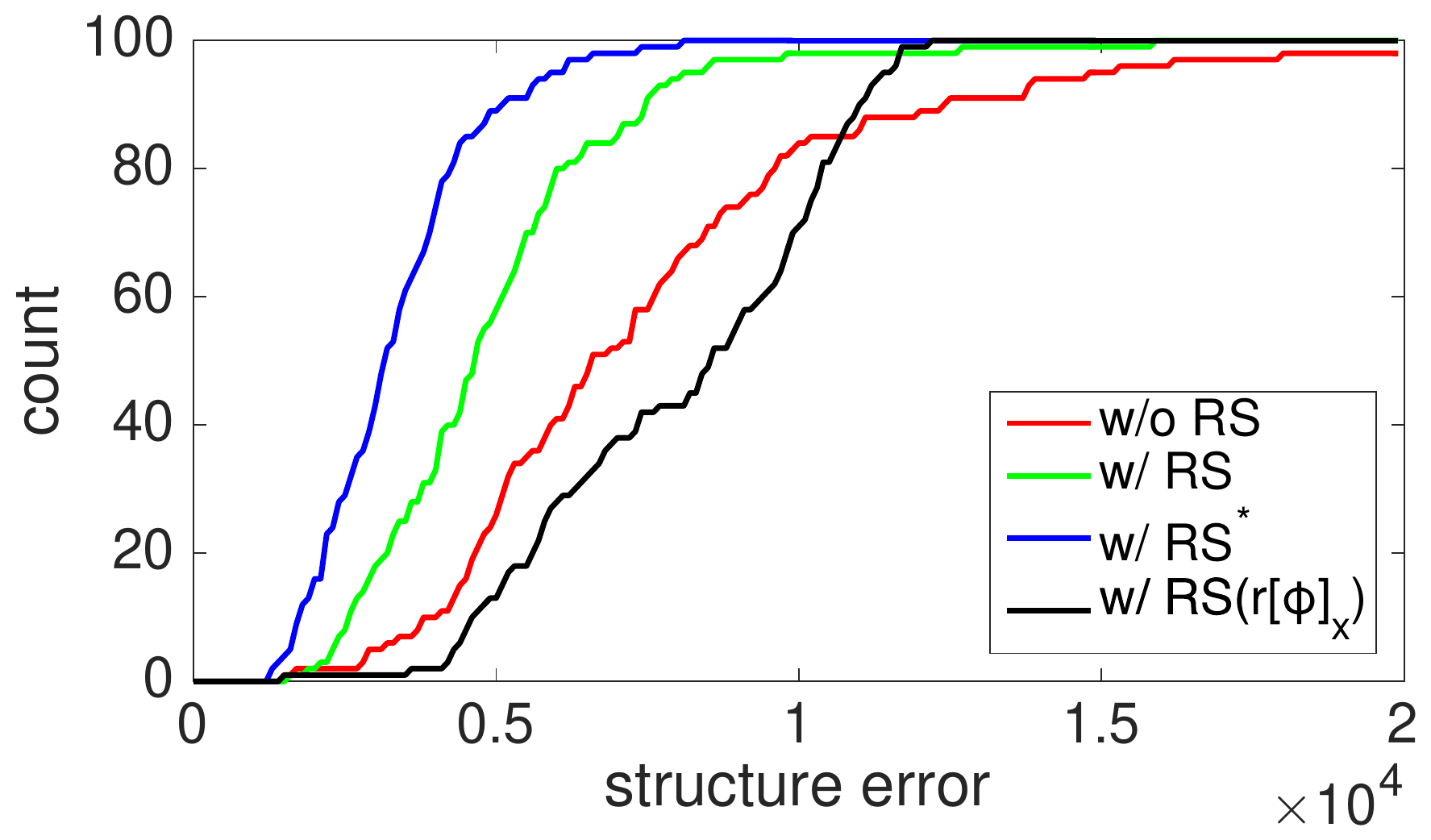}
\hspace*{-4mm}&\includegraphics[width=43.0mm]{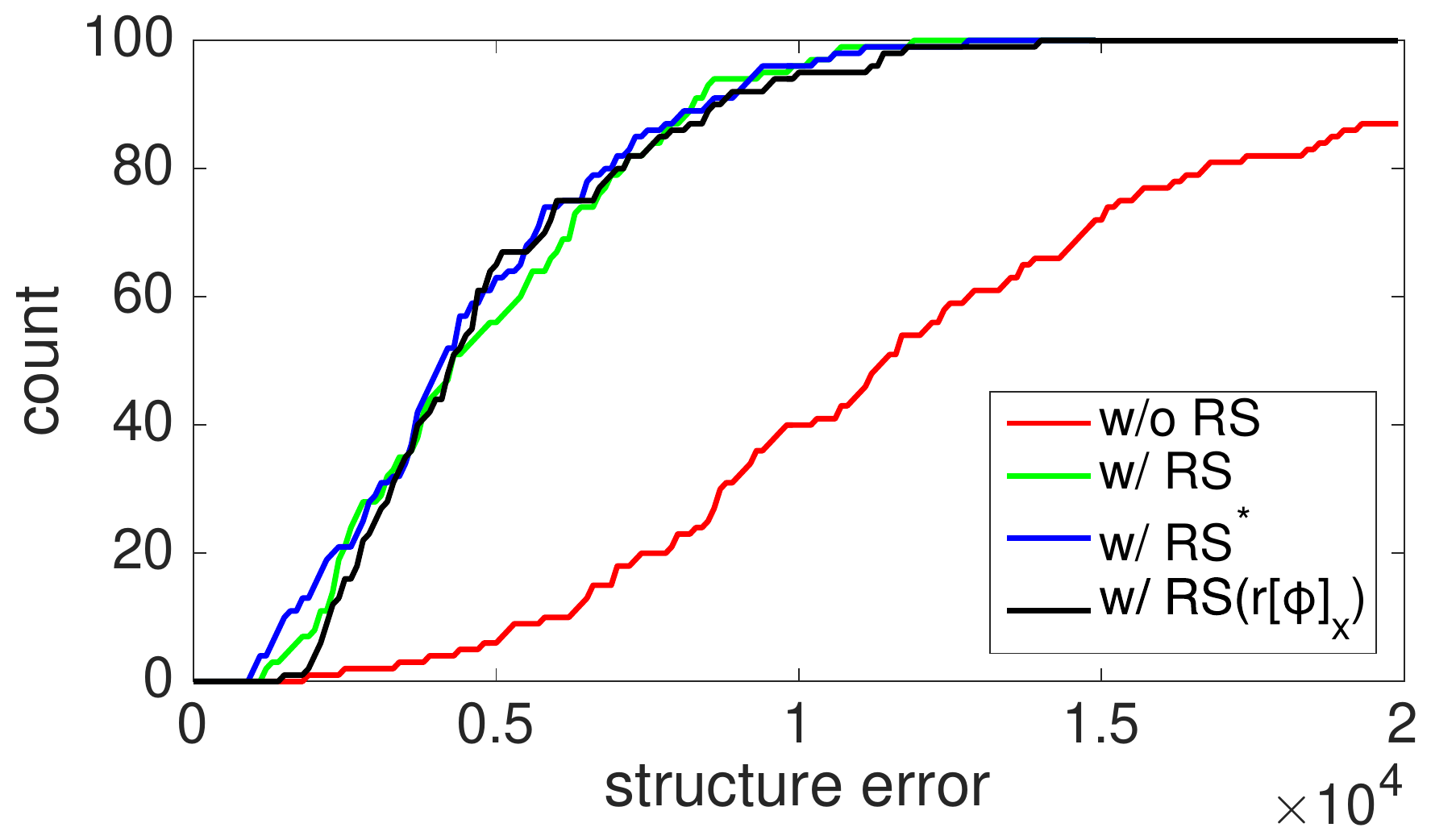}
\hspace*{-4mm}&\includegraphics[width=43.0mm]{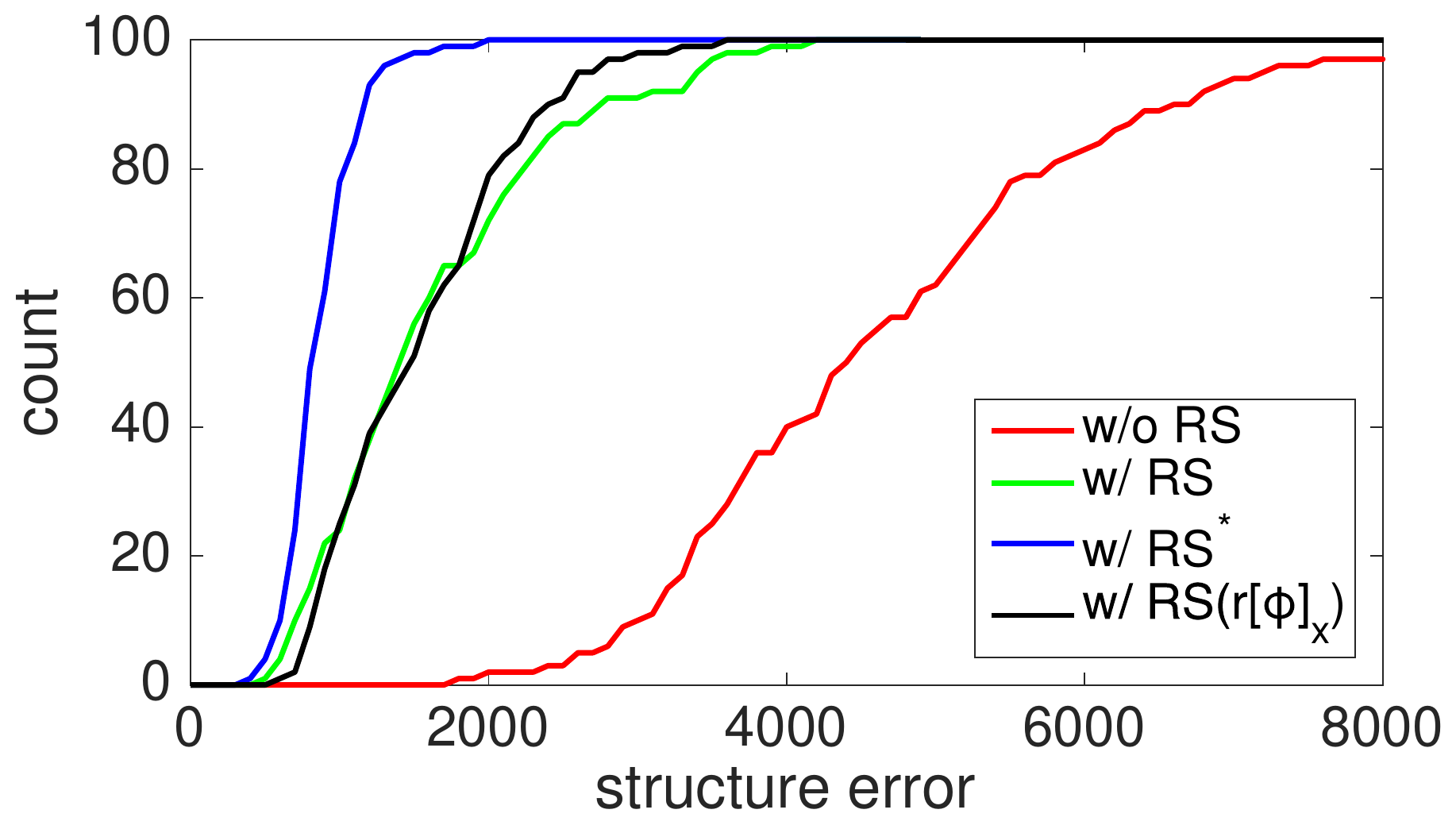}
\hspace*{-4mm}&\includegraphics[width=43.0mm]{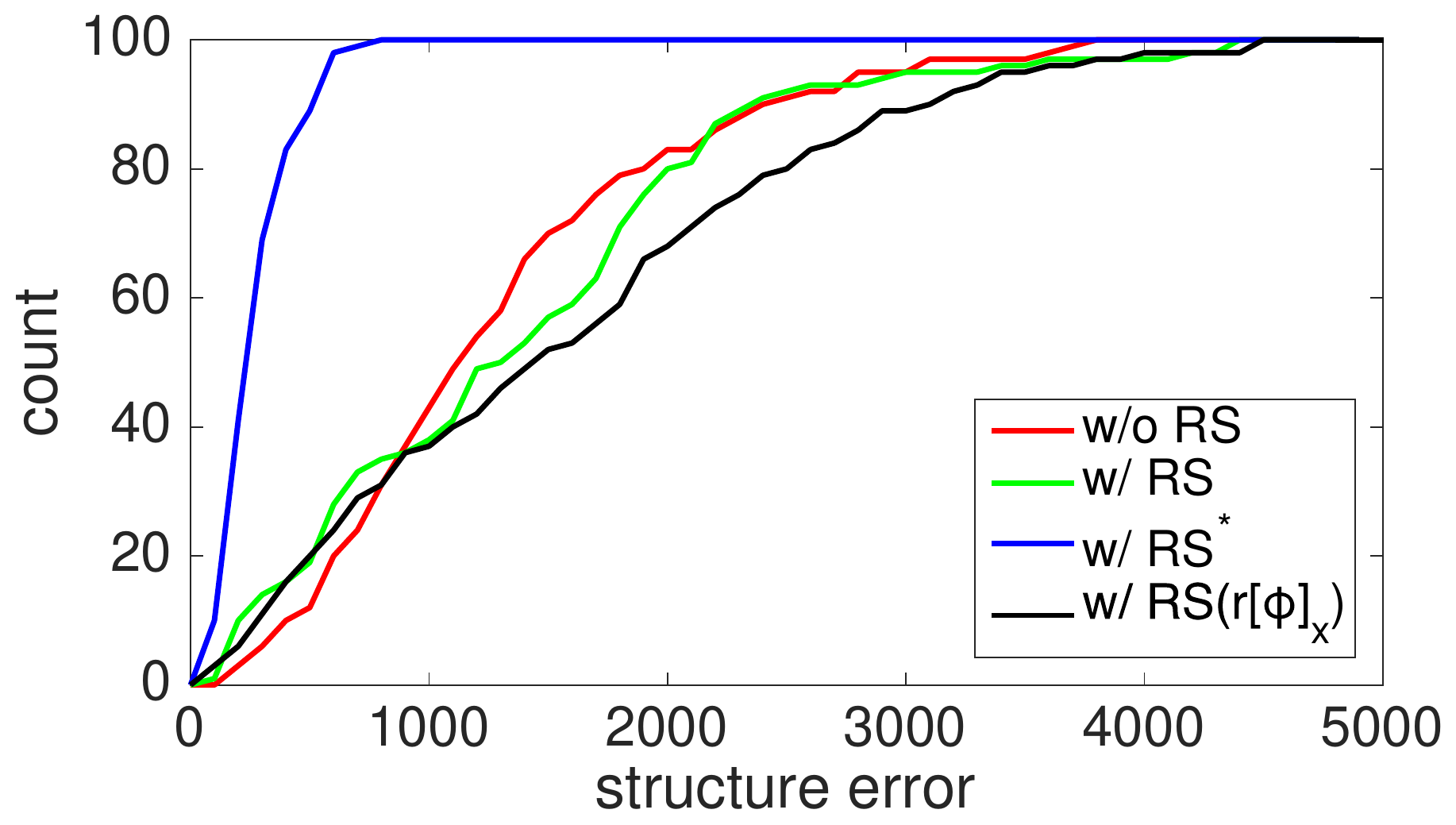}\\

~~~~~~~fountain-P11& ~~~~~Herz-Jesu-P8& ~~~~~castle-P19& ~~~Temple
\end{tabular}

\caption{Results of Experiment I. Cumulative histogram of errors of
  estimated rotation and translation components of camera poses and of
  estimated points. ``w/o RS'' is ordinary BA without
  the proposed RS camera model. ``w/ RS'' and ``w/ RS*'' are BA incorporating the proposed RS model; 
  $\phi^1_1$ is fixed to 0 for the latter. ``w/ RS$(r[\phi]_\times])$'' is BA incorporating the original RS model (\ref{eqn:distantcase}).}
\label{fig:pointexp}
\end{figure*}

Several observations can be made from \Fig{pointexp}. First, it is
observed that the method ``w/ RS*'' that fixes $\phi_1^1$ shows the
best performance for all the sequences. The method ``w/ RS'' rivals
``w/ RS*''' only for Herz-Jesu-P8. This implies that the camera motion
of fountain-P11, castle-P19, and Temple are close to CMSs, whereas
that of Herz-Jesu-P8 is distant from CMSs. In fact, camera
orientations in the former three sequences tend to have small
elevation angles, which is one of the conditions for camera motion to
be the CMS discussed in Sec.\ref{sec:cmsrep}, while those for
Herz-Jesu-P8 tend to have larger elevation angles. Second, it is seen
that ``w/ RS$(r[\phi]_\times])$'' (exact RS model) shows similar
  behaviors to ``w/ RS''. This validates the approximation we used to
  derive our RS model (Sec.\ref{sec:affine}), i.e., affine camera
  approximation and first-order approximation with respect to
  $(\phi_1, \phi_2, \phi_3)$. 

We also show typical reconstruction results for the four sequences in
Figs.\ref{fig:seq1}-\ref{fig:seq7}. It is seen that for each sequence,
the method ``w/ RS*,'' consistently yields the most accurate camera
path and point cloud than others. It is also seen that the method ``w/
RS'' tends to yield structure that is elongated vertically, which
explains the large structure error shown in the cumulative histograms
of \Fig{pointexp}. This structure elongation is well explained by the
CMS described in Proposition \ref{prop3}.
\begin{figure}[bth]
\tiny
\renewcommand{\arraystretch}{0.5}
\includegraphics[width=80mm]{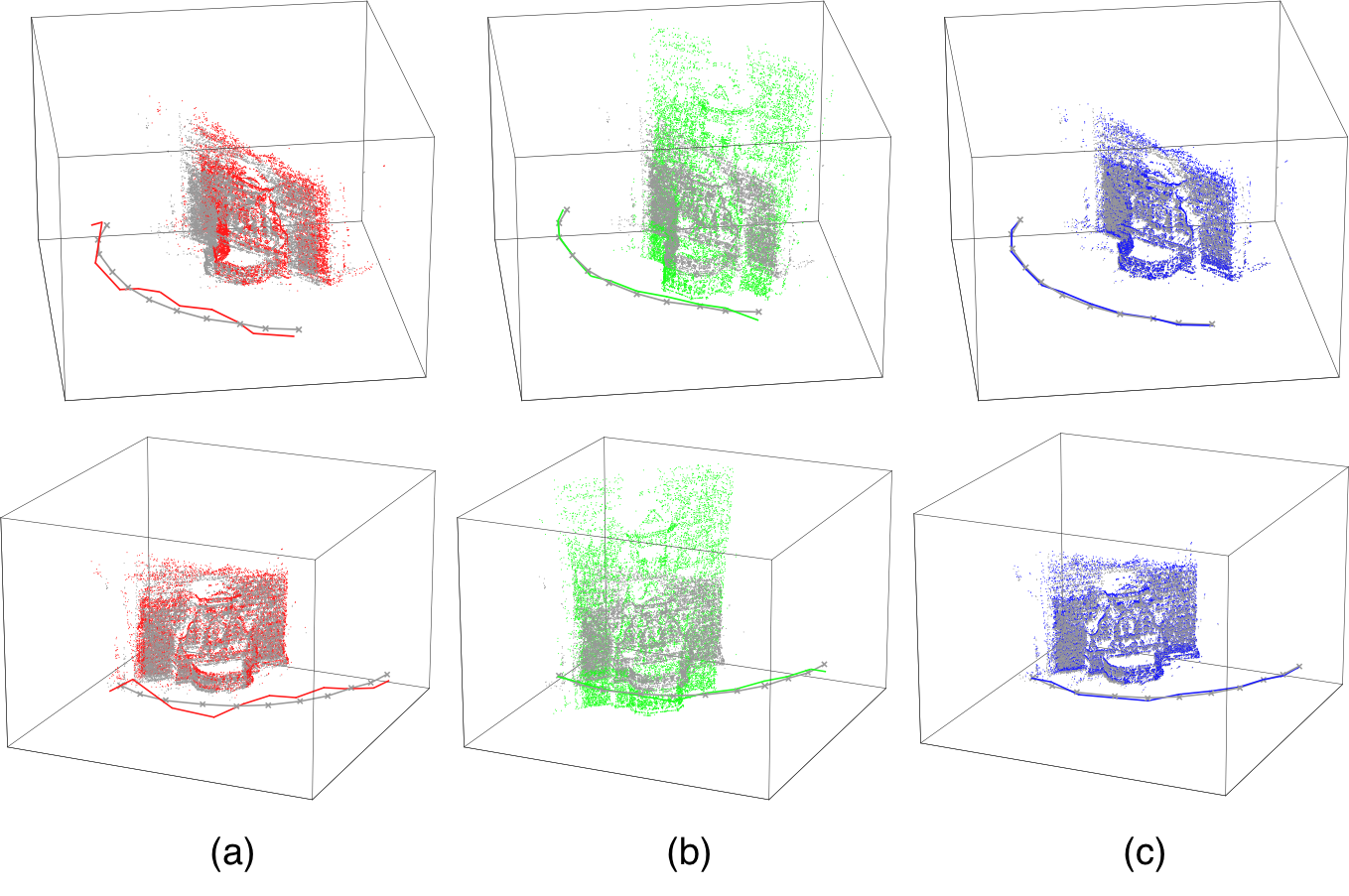}

  \caption{Typical reconstruction results for sequence
    fountain-P11. (a) w/o RS. (b) w/ RS. (c) w/ RS*
  ($\phi_1^1$ fixed). Grey dots and lines with crosses are true scene
  points and true camera positions, respectively. }
  \label{fig:seq1}
\end{figure}

\begin{figure}[bth]
\tiny
\renewcommand{\arraystretch}{0.5}
\includegraphics[width=80mm]{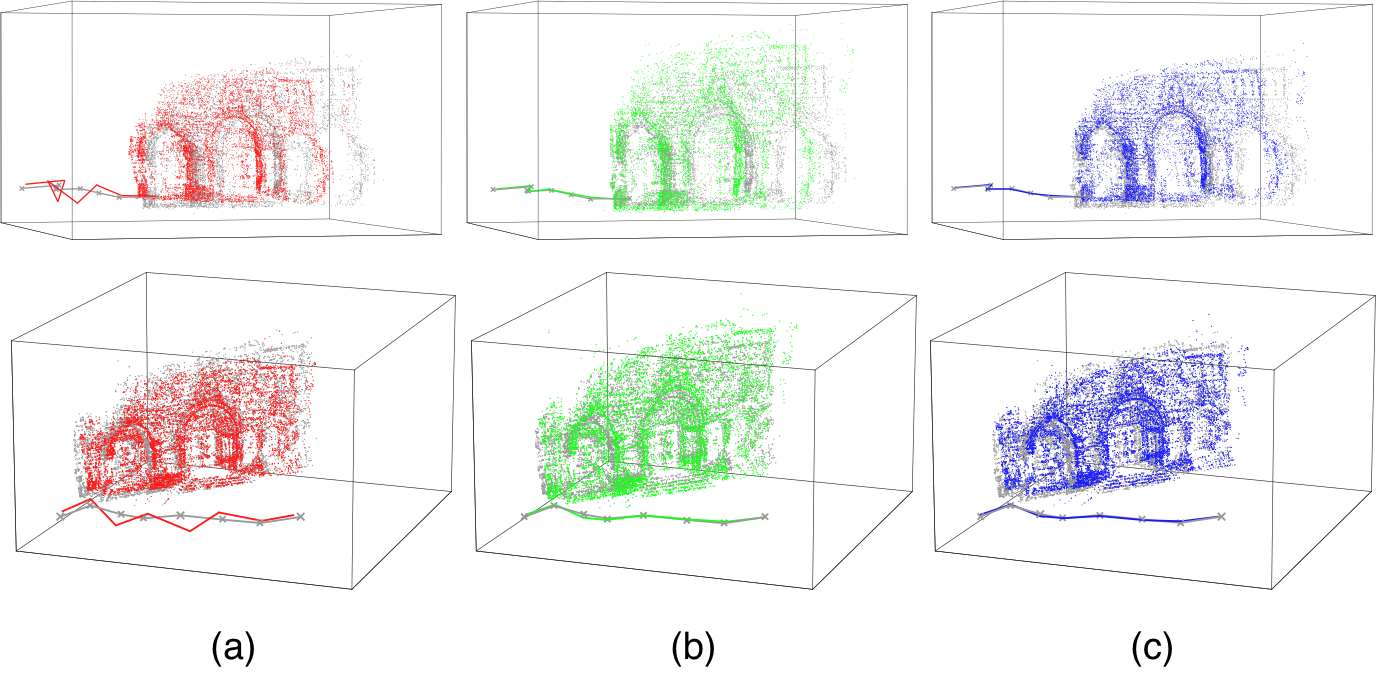}

  \caption{Typical reconstruction results for sequence
    Herz-Jesu-P8. (a) w/o RS. (b) w/ RS. (c) w/ RS*
  ($\phi_1^1$ fixed).}
  \label{fig:seq2}
\end{figure}

\begin{figure}[bth]
\tiny
\renewcommand{\arraystretch}{0.5}
\includegraphics[width=80mm]{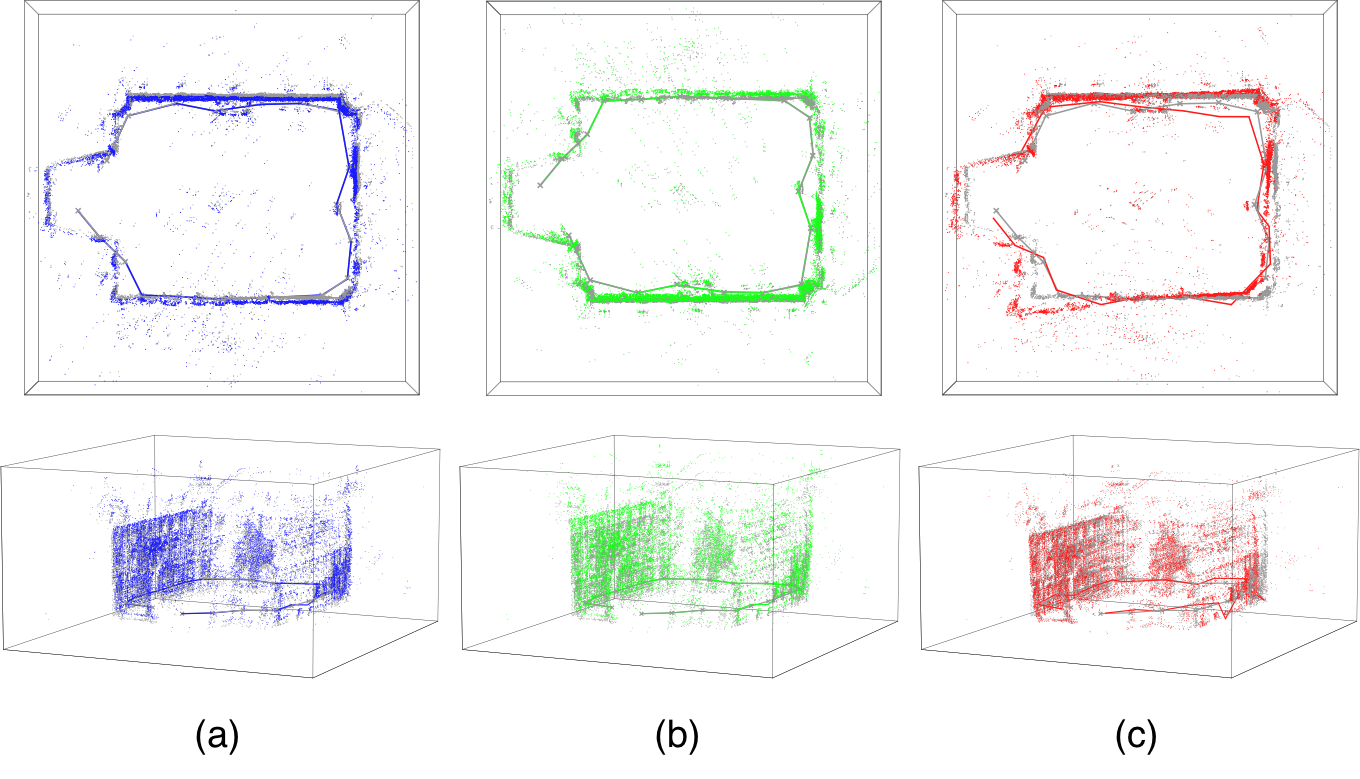}

  \caption{Typical reconstruction results for sequence
    castle-P19. (a) w/o RS. (b) w/ RS. (c) w/ RS*
  ($\phi_1^1$ fixed). }
  \label{fig:seq5}
\end{figure}

\begin{figure}[bth]
\tiny
\renewcommand{\arraystretch}{0.5}
\includegraphics[width=80mm]{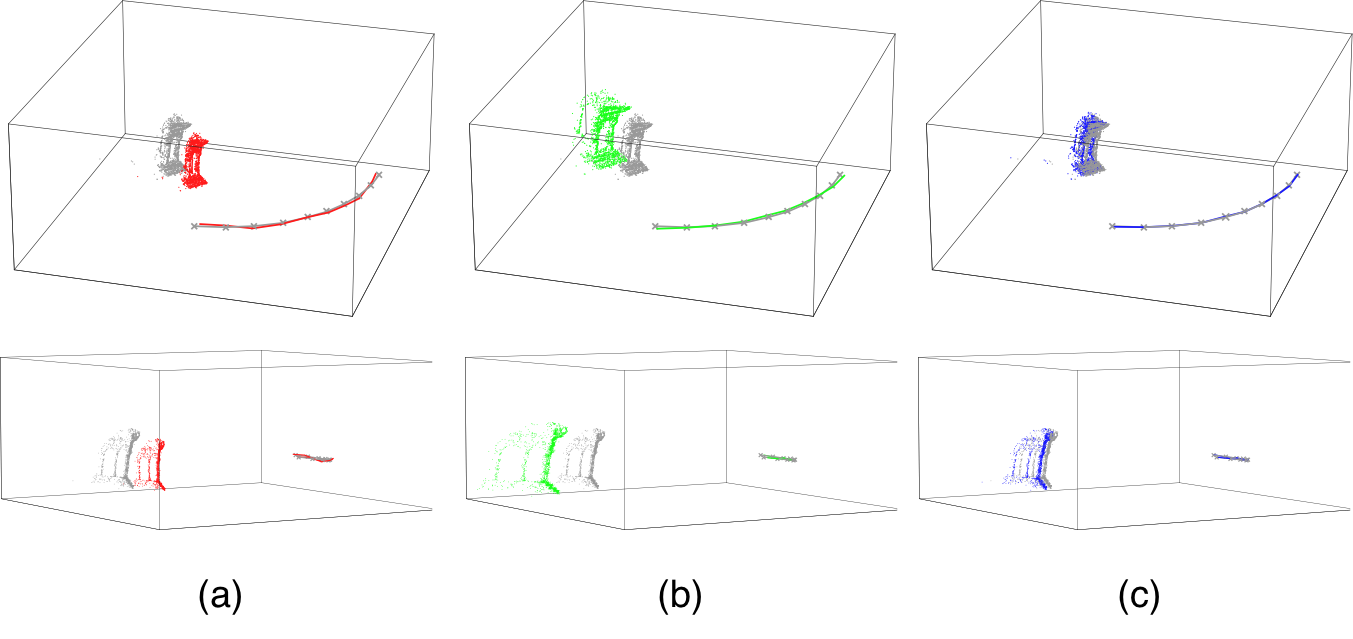}

  \caption{Typical reconstruction results for sequence
    Temple. (a) w/o RS. (b) w/ RS. (c) w/ RS*
  ($\phi_1^1$ fixed).}
  \label{fig:seq7}
\end{figure}


\subsection{Real image experiments}

We conducted experiments using images captured by the camera of a
smartphone (iPhone 5s). We acquired still-images of a scene from a
number of different viewpoints by hand-holding the smartphone. We took
two images at each viewpoint, one by firmly holding the smartphone
with both hands, and the other by deliberately swinging the camera a
little during frame capture. As the smartphone was hand-held, its
orientation cannot be precisely controlled, whereas its position
should not change much between the two acquisitions; the change will
be less than 20cm at each viewpoint. Each scene contains a building
with the length 20-30m, and thus possible position changes will be
negligibly small as compared with the building size. Thus, we use
differences in the camera positions that are estimated by VisualSFM
from these two image sequences as error measure. To be specific,
regarding the camera positions recovered from the distortion-free
sequence as the ground truth, we measure errors of those recovered
from the other sequence with RS distortion. We did not prune any
potentially incorrect matches. We applied two methods, one with the
proposed RS model (with $\phi^1_1$ fixed assuming the first image to
be distortion-free) and the other without the proposed model, and
compared their accuracy. Figure \ref{fig:shrine1} shows results for
two different scenes Shrine1 and Shrine2. (The original images of the
two scenes are shown in the supplementary material.) These results
demonstrate the validity of our simplified RS model in spite of the
employed approximations to derive it.


\begin{figure}[t]
\footnotesize 

\hfil
\hbox{
\includegraphics[width=30mm]{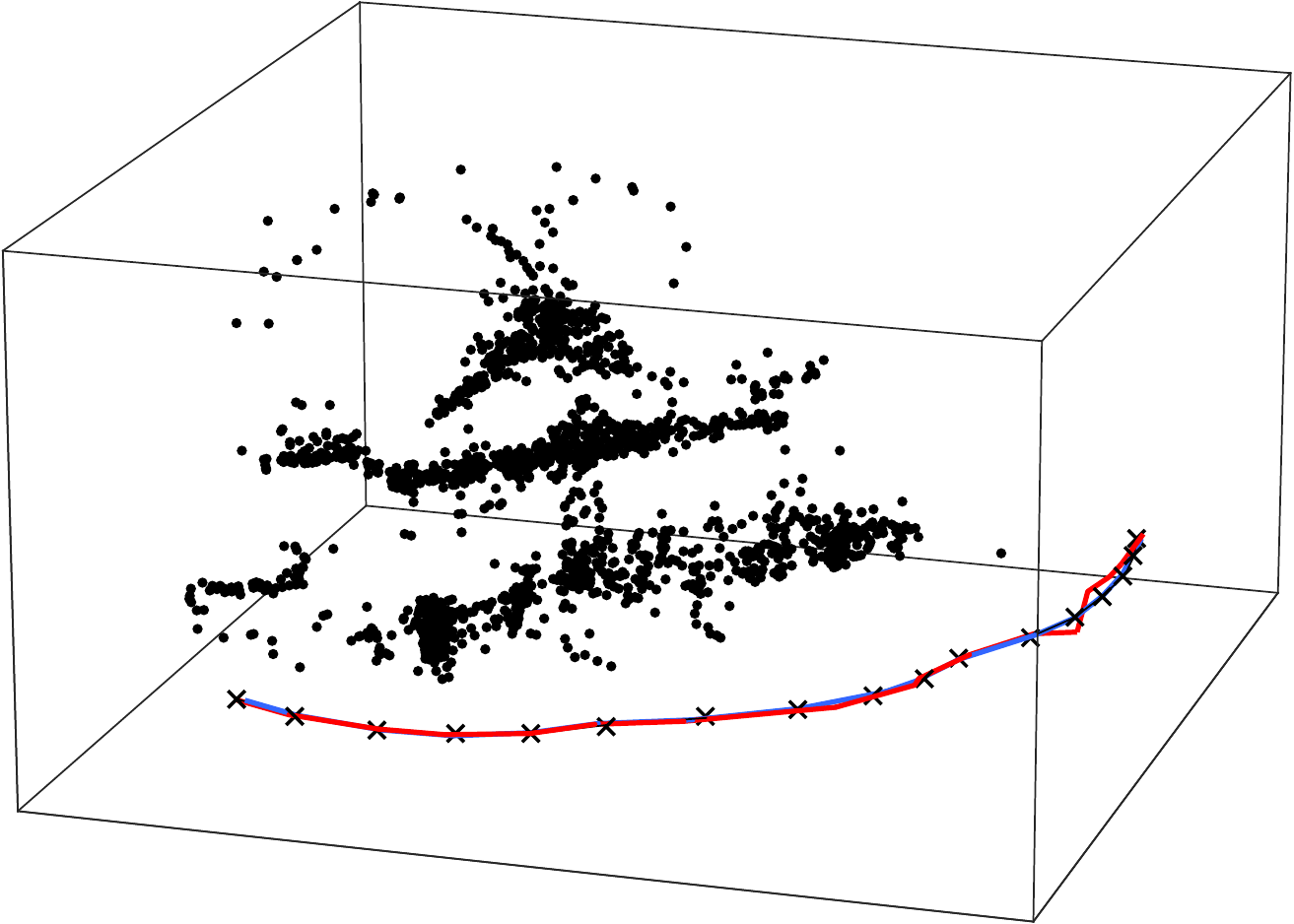}
~~~~~\includegraphics[width=39.5mm]{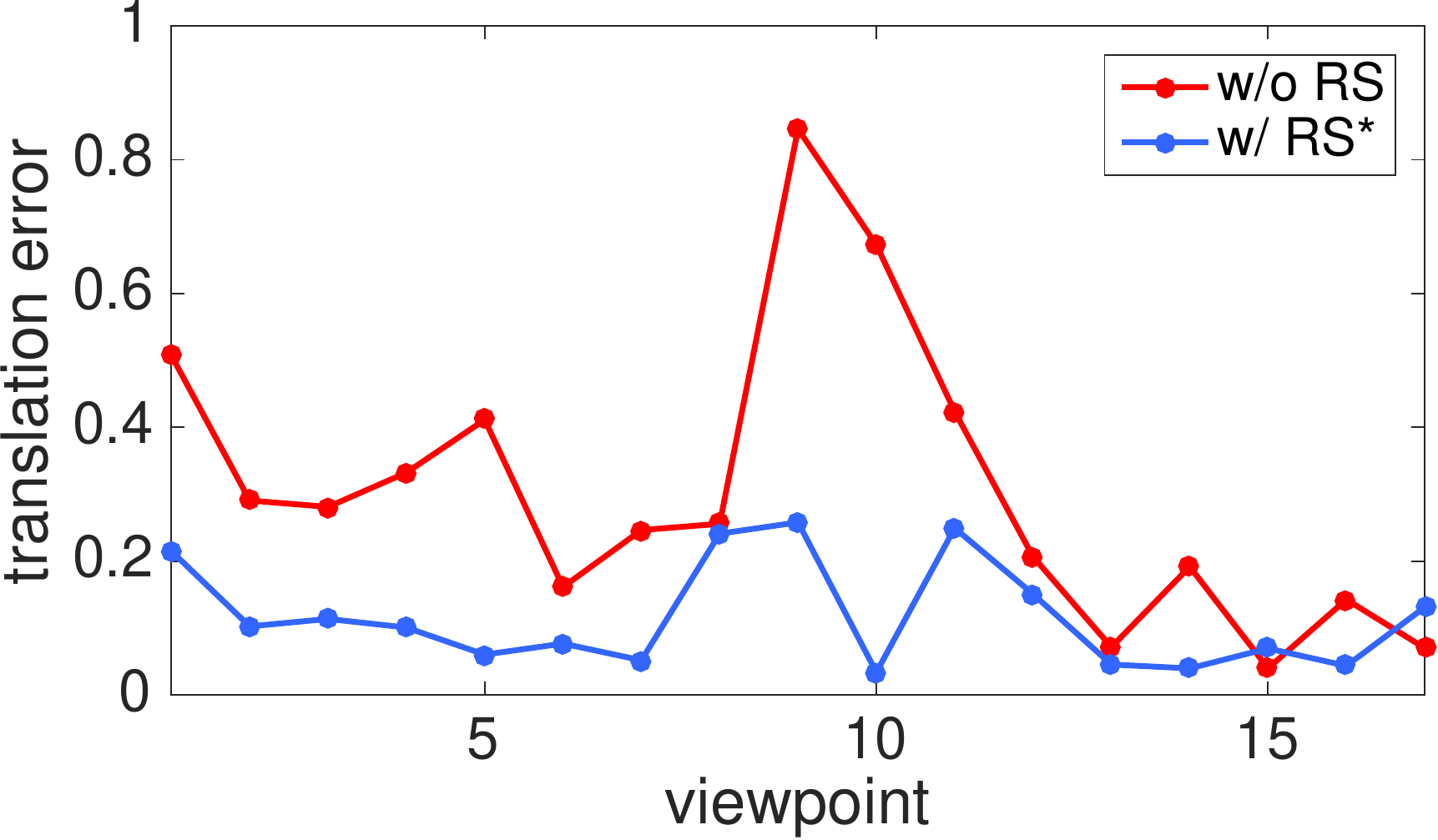}}

\vspace*{2mm}
\hfil
\hbox{
\raisebox{1mm}{\includegraphics[width=31mm]{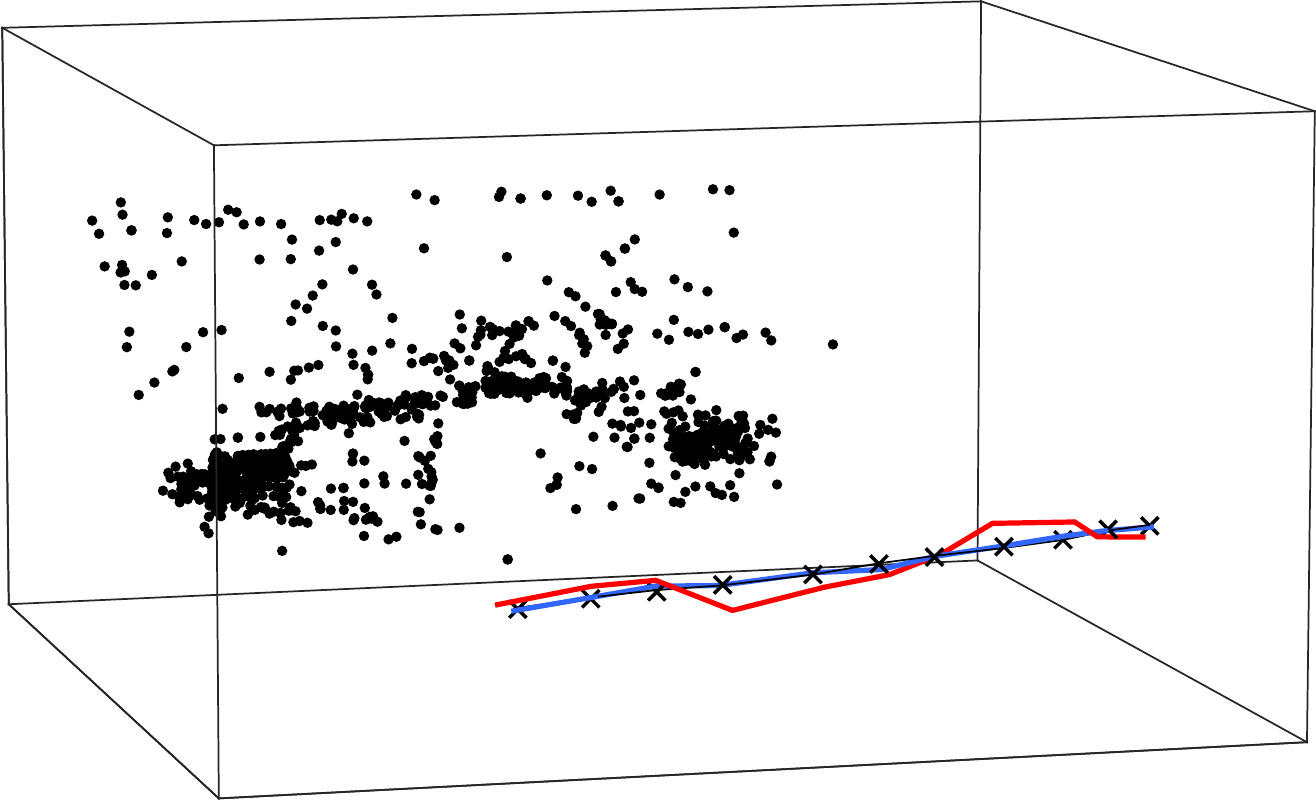}}
~~~~\includegraphics[width=39mm]{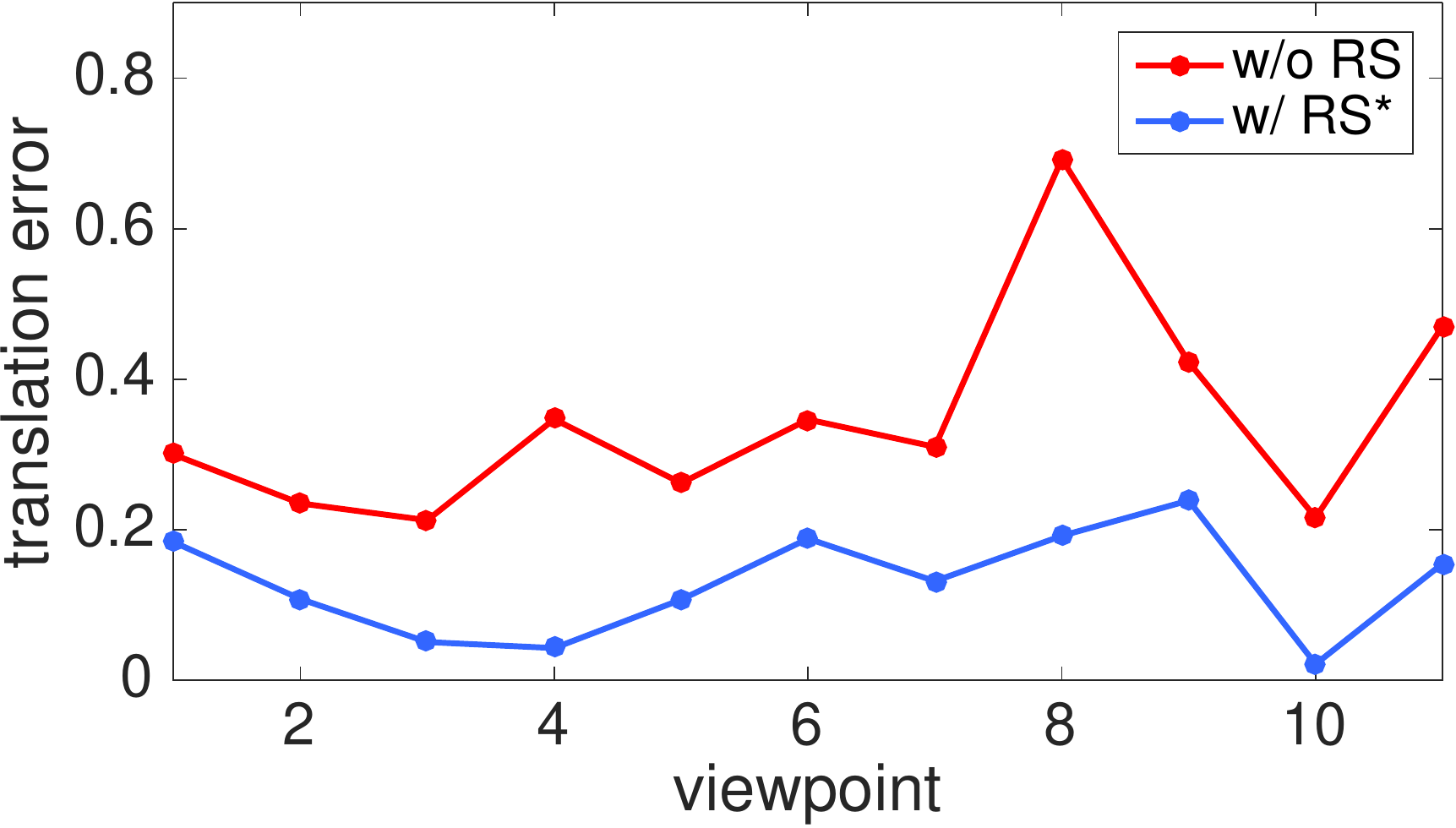}}


\vspace*{2mm}
  \caption{Two results of reconstruction from real images. The camera
    positions obtained from distortion-free images (assumed to be
    ground truths) are shown as black x marks. Those obtained by the
    proposed approach are in blue and those by BA without a RS model
    are in red. Distances from the estimated camera positions to those
    computed from distortion-free images. }
\label{fig:shrine1}
\end{figure}



\section{Summary and conclusions}

We have discussed degeneracy of rolling shutter SfM. Assuming
rotation-only camera motion with linearized rotation, we have shown
that, employing the affine camera approximation (i.e., the first-order
approximation of the perspective effect on the RS distortion), the RS
distortion can be represented by a composition of the following two
transformations: i) two-dimensional projective transformation with two
parameters that can be interpreted as skew and aspect ratio of an
imaginary camera, and ii) one-parameter nonlinear transformation that
can be interpreted as a type of lens distortion of the
camera. Assuming this approximate RS camera model, we have shown that
the problem can be recasted as self-calibration of the imaginary
camera, in which skew and aspect ratio are unknown and varying in the
image sequence. For this self-calibration problem, we have derived a
general representation of CMSs, and also shown a practically important
CMS that was previously reported in the literature. As ambiguity with
the latter CMS can be resolved by specifying only one of the two RS
distortion parameters of a single viewpoint, we have proposed to
identify an image in the sequence that undergoes no distortion and set
these parameters to zero. The experimental results demonstrate the
effectiveness of the proposed approach to CMSs of RS SfM.



{\small
\bibliographystyle{ieee}
\bibliography{rollingshutter}
}

\end{document}